%% file: main.tex
\documentclass{article} 
\usepackage{iclr2027_conference,times}
\usepackage{graphicx}
\usepackage{subcaption}
\usepackage{float}
\usepackage{algorithm}
\usepackage{algorithmic}
\usepackage{amsmath,amssymb}
\usepackage{amsthm}

\newtheorem{theorem}{Theorem}
\newtheorem{proposition}{Proposition}

\input{math_commands.tex}

 \usepackage{caption} 
\usepackage{hyperref}
\usepackage{url}
\usepackage{tabularx}
\usepackage{booktabs,multirow,threeparttable,graphicx}
\title{Latency and accuracy tradeoffs in Spiking Neural Networks}

\author{Zhanglu Yan$^{1}$, Zixuan Zhu$^{1,3,4}$, Kaiwen Tang$^{1}$,Yuyang Cai$^{2}$, Qianhui Liu$^{2}$, Weng-Fai Wong$^{1}$\\[0.5em]
\normalfont\small $^{1}$Department of computer science, National University of Singapore\\
\normalfont\small $^{2}$School of Artificial Intelligence, Shandong University\\
\normalfont\small $^{3}$Shanghai Advanced Research Institute, Chinese Academy of Science\\
\normalfont\small $^{4}$University of Chinese Academy of Sciences, Beijing
}

\usepackage{etoolbox}
\makeatletter
\patchcmd{\@maketitle}
  {Under review as a conference paper at ICLR 2027}
  {Preprint}{}{\errmessage{Failed to replace the review header}}
\patchcmd{\@maketitle}
  {Anonymous authors\\Paper under double-blind review}
  {\@author}{}{\errmessage{Failed to display the authors}}
\makeatother

\renewcommand{\iclrruler}[1]{}
\begin{document}

\maketitle

\begin{abstract}

Spiking neural networks (SNNs) are attractive for low-power speech command recognition, yet their latency has received far less attention than their energy efficiency, and their multi-timestep execution is widely assumed to make them slower than quantized neural networks (QNNs). This paper challenges the assumption that more local timesteps necessarily imply higher network latency. By overlapping computation across adjacent layers at the timestep level, SNNs may complete execution in less time than comparable bit-serial QNNs. However, this overlap relies on spikes firing on incomplete inputs, and a spike once generated cannot be withdrawn, so its error persists and reduces accuracy. Waiting for more input before firing would seem to improve accuracy at the cost of reduced overlap. Yet we find and prove that this intuition fails at some layers, where even a small increase in waiting can change spike timing and downstream computation, making the network both slower and less accurate. 
We therefore propose a Pipeline Delay Search (PDS) method which selects each layer’s delay by balancing task-level accuracy gains against added network latency. We then adapt the selected configurations through spike-based quantization-aware training and bounded tuning of firing thresholds and initial membrane potentials. Together, these steps form Falcon, a framework for \underline{F}ine-grained \underline{A}nalysis of \underline{L}atency and \underline{CON}trolled firing which systematically analyzes and optimizes SNN latency under a spatial analog compute-in-memory mapping with shared digital engines.
We evaluate Falcon on Google Speech Commands V2 (GSC) and Spiking Speech Commands (SSC), achieving competitive accuracies of 96.31\% and 83.02\% at modeled network-core latencies of 119.64 and 124.00$\mu\mathrm{s}$, respectively. Together, our analysis and results show that SNNs can compute more yet finish faster, and wait longer yet predict worse, highlighting why Falcon matters for both latency and accuracy.

\end{abstract}

\section{Introduction}
\label{sec:intro}

Spiking neural networks (SNNs) have attracted growing interest in
low-power keyword spotting and speech command recognition~\citep{wang2024efficient,yang2022deep}.
Their binary spike communication and event-driven computation offer
opportunities to reduce computation and data-movement energy relative
to quantized neural networks (QNNs)~\cite{11631730}.
Recent models, including SpikeSCR, SpikCommander, and
SIDC-KWS~\citep{wang2024efficient,wang2026spikcommander,lim2025sidc},
report competitive recognition accuracy and promising energy efficiency.
However, existing studies focus primarily on accuracy and energy efficiency, while their discussion of latency is largely limited to time-window length, leaving it unclear whether SNNs offer faster inference than QNNs.

In terms of processing rounds, SNNs seem to be at a disadvantage. For example, representing eight levels requires three bit-serial rounds in QNN but seven timesteps in a rated-based SNN. Assuming equal per-round latency, the SNN therefore takes longer to complete each layer. However, this does not necessarily increase end-to-end latency, which also depends on when the next layer can start. In the bit-serial QNN considered here, each output becomes available only after all three bit contributions are accumulated and the result is quantized. In contrast, integrate-and-fire (IF) neurons in SNN can emit and propagate spikes after processing only part of the input, without waiting for all seven timesteps to finish. Showing in Figure~\ref{fig:pipeline_intro}(a), the next layer can therefore process earlier timesteps while the preceding layer processes later ones. This cross-timestep pipelining, enabled by IF neurons, reduces inter-layer waiting, potentially allowing SNNs to achieve faster inference than the bit-serial QNN baseline despite longer per-layer latency.

\begin{figure}[ht]
    \centering
    \includegraphics[width=0.9\linewidth]{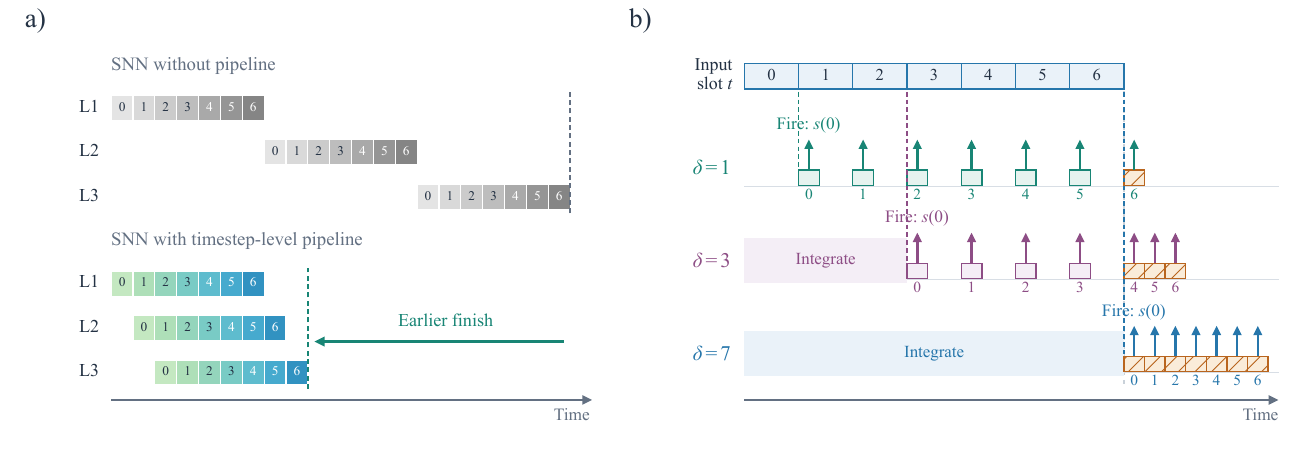}
\caption{Cross-timestep pipelining and firing delay in SNNs.
(a) Different layers process different timesteps in parallel,
allowing earlier completion.
(b) Delay-controlled IF kernel with $\delta$ sets how many input slots are accumulated before the first firing decision.
}
\label{fig:pipeline_intro}
\end{figure}

However, this overlap can reduce accuracy because neurons may fire before all inputs arrive. Later inputs may cancel earlier contributions, but spikes already sent cannot be taken back and may already have affected the next layer. Waiting for more inputs therefore seems to offer a simple trade-off: larger latency but higher accuracy. However, we find and prove that preserving spike counts does not guarantee the same prediction: spike timing can change whether downstream neurons cross the firing threshold. Waiting longer can therefore make inference both slower and less accurate. Effective pipelining requires choosing which layers should wait, and for how long.

To address this challenge, we introduce Falcon, a framework for
\underline{F}ine-grained \underline{A}nalysis of \underline{L}atency
and \underline{CON}trolled firing.
Falcon introduces a delay-controlled IF kernel (Figure~\ref{fig:pipeline_intro}(b)) that controls how much
input is accumulated before each firing decision, enabling cross-timestep pipelining across convolutional, linear, and residual stages.
Because suitable delays depend on the model and data, we propose
Pipeline Delay Search (PDS) to choose where and how long to wait.
With model parameters fixed, PDS evaluates layer-wise delay changes
using validation accuracy and modeled core latency.
Among the schedules found within the matched QNN latency budget,
it selects \textit{Fastest} for minimum latency, \textit{Accurate}
for maximum validation accuracy, and \textit{Balanced} by maximizing normalized accuracy minus normalized latency.
With the selected delays fixed, we adapt the network through
quantization-aware training with a spiking forward pass and bounded
tuning of firing thresholds and initial membrane potentials.
Finally, we evaluate core inference latency under a spatial analog-CIM mapping with shared digital engines. We consider two model sizes, Falcon-Medium and Falcon-Large, with 3 and 5 Transformer blocks, respectively.
Analog arrays execute static-weight convolutions and linear projections, while digital units handle input-dependent attention, neuron updates,
and residual additions. We compute latency from analog readout and digital
execution cycles, accounting for input dependencies, inter-layer
waiting, and computation overlap.

We evaluate Falcon on Google Speech Commands v2 (GSC) and Spiking
Speech Commands (SSC)~\citep{warden2018speech,cramer2022heidelberg},
comparing it with matched QNNs, Full-lookahead SNNs, and prior SNN methods.
In the Full-lookahead setting, FALCON-Large achieves $96.91\%$
test accuracy on GSC, while FALCON-Medium achieves
$83.94\%$ on SSC.
For low-latency inference, Falcon-Medium with \textit{Balanced} schedule
achieves $96.31\%$ accuracy on GSC at a modeled core latency of
$119.64\,\mu\mathrm{s}$ and $83.02\%$ accuracy on SSC with latency of $124\,\mu\mathrm{s}$. These results show that Falcon offers flexible accuracy--latency trade-offs, supporting both high-accuracy and low-latency SNN inference
through delay-controlled firing.

\section{Preliminaries}
\label{sec:preliminaries}

\paragraph{Analog-CIM Execution Model:}

Analog compute-in-memory (CIM) performs vector--matrix multiplication
within memory arrays.
Weights are stored as cell physical states, and input drive
the array rows.
The resulting currents add along each column to form weighted sums,
which analog-to-digital converters (ADCs) convert into digital codes~\cite{8910375,ye202328}.
We define one \emph{analog cycle} as one row-group evaluation and
ADC conversion of the selected column outputs
$
t_{\mathrm{cycle}}
=
t_{\mathrm{set}}+t_{\mathrm{ADC}},
$
, where $t_{\mathrm{set}}$ covers input setup and array settling,
and $t_{\mathrm{ADC}}$ is the time to complete one ADC conversion.
Processing one input vector may require multiple analog cycles.
Each crossbar tile contains $R_{\mathrm{tile}}$ rows, and
$N_{\mathrm{active}}$ rows can be activated at once.
Shared ADCs also read the columns in $F_{\mathrm{MUX}}$ sequential
groups. For an input dimension $K_l$, with larger inputs split
across parallel tiles, one complete analog pass requires
\[
N_{\mathrm{cycle},l}
=
\underbrace{
\left\lceil
\frac{\min(K_l,R_{\mathrm{tile}})}{N_{\mathrm{active}}}
\right\rceil
}_{\text{row-activation groups}}
\underbrace{F_{\mathrm{MUX}}}_{\text{column-readout groups}},
\qquad
\tau_l=N_{\mathrm{cycle},l}t_{\mathrm{cycle}}.
\]
Here, $\tau_l$ is the time required for one analog pass.
Each row-group and column-group combination incurs one analog cycle.

\paragraph{Local Cycles of QNNs and SNNs:}

We consider bit-serial QNN execution and SNN activations represented
by spike counts.
A $b_l$-bit activation has $2^{b_l}$ levels, while $T_l$ binary spike
slots represent counts from $0$ to $T_l$.
Matching the number of levels gives $T_l=2^{b_l}-1$.
Under the same weights, array mapping, and ADC precision,
each bit plane or spike slot requires one analog pass.
The local analog cycle counts are therefore
$
C_{\mathrm Q}^{(l)}=b_lN_{\mathrm{cycle},l}$ and 
$
C_{\mathrm S}^{(l)}=T_lN_{\mathrm{cycle},l},
$
with execution times $b_l\tau_l$ and $T_l\tau_l$, respectively.

\section{Methods}
\label{sec:methods}

\subsection{Delay-controlled IF kernel }
\label{sec:pipelined_execution}
Falcon enables cross-timestep pipelining through a delay-controlled integrate-and-fire (IF) kernel. A layer-specific delay determines how many input timesteps contribute
to each firing decision. Each layer forwards its output spikes as they are generated, allowing downstream layers to process earlier timesteps while upstream layers continue processing later ones.

We first introduce the Delay-controlled IF kernel $\mathrm{SN}_{\delta}$.
We assume that each IF layer processes T input slots and makes T firing decisions.
Given input spikes $s_i^{l-1}(t)\in\{0,1\}$, the normalized current
to output neuron $j$ in layer $l$ is
\begin{equation}
I_j^l(t)
=
\gamma_j^l\sum_i w_{ij}^l s_i^{l-1}(t),
\label{eq:Falcon_slot_current}
\end{equation}
where $w_{ij}^l$ is the synaptic weight and $\gamma_j^l$ is a fixed
output scaling factor.
A normalized static term $I_{j,0}^l$ collects input-independent
contributions and is distributed equally across the $T$ input slots.
The membrane potential is initialized to
$V_j^l=\frac12+\mu_j^l$, where $\mu_j^l$ is the initial offset.

For output slot $t$, the firing delay $\delta^l\in\{1,\ldots,T\}$
specifies the last required input slot:
\begin{equation}
t_{\mathrm e}
=
\min(t+\delta^l-1,T-1).
\label{eq:Falcon_delay}
\end{equation}
The index $t_{\mathrm s}$, initialized to zero, tracks the next
unprocessed input slot.
Before decision $t$, the neuron processes the required inputs in order.
While $t_{\mathrm s}\leq t_{\mathrm e}$, it waits for input slot
$t_{\mathrm s}$ and updates
\begin{equation}
V_j^l
\leftarrow
V_j^l+\frac{I_{j,0}^l}{T}+I_j^l(t_{\mathrm s}),
\qquad
t_{\mathrm s}\leftarrow t_{\mathrm s}+1.
\label{eq:Falcon_if_receive}
\end{equation}
Once these inputs have been accumulated, the neuron makes its firing
decision and applies soft reset:
\begin{equation}
s_j^l(t)=\mathbf{1}[V_j^l\geq\theta_j^l],
\qquad
V_j^l\leftarrow V_j^l-\theta_j^l s_j^l(t),
\label{eq:Falcon_if_fire}
\end{equation}
where $\theta_j^l>0$ is the firing threshold.
The neuron has no leakage and emits at most one spike per output slot. With $\delta^l=1$, each input slot is followed by a firing decision;
with $\delta^l=T$, all inputs are accumulated before the first decision.
After all inputs have been processed, the remaining firing decisions
use the stored membrane potential without adding further current.
Output spikes are passed to downstream operators as they are generated.
Algorithm~\ref{alg:rate_if} summarizes $\mathrm{SN}_{\delta^l}$.

\begin{algorithm}[ht]
\caption{Delay-controlled Rate-IF kernel $\mathrm{SN}_{\delta^l}$}
\label{alg:rate_if}
\begin{algorithmic}[1]
\REQUIRE
Input spike slots
$\{\mathbf{s}^{l-1}(t)\}_{t=0}^{T-1}$;
current map in Eq.~\ref{eq:Falcon_slot_current};
static term $\mathbf I_0^l$;
delay $\delta^l\in\{1,\ldots,T\}$;
thresholds $\boldsymbol\theta^l>0$;
initial offsets $\boldsymbol\mu^l$.
\ENSURE
Registered output slots
$\{\mathbf{s}^l(t)\}_{t=0}^{T-1}$.
\STATE
$\mathbf V^l\gets\frac12\mathbf1+\boldsymbol\mu^l$;
$t_{\mathrm s}\gets0$.
\FOR{$t=0$ to $T-1$}
    \STATE
    $t_{\mathrm e}\gets\min(t+\delta^l-1,T-1)$.
    \WHILE{$t_{\mathrm s}\leq t_{\mathrm e}$}
        \STATE Wait until input slot $t_{\mathrm s}$ is available.
        \STATE Compute $\mathbf I^l(t_{\mathrm s})$ using
        Eq.~\ref{eq:Falcon_slot_current}.
        \STATE
        $\mathbf V^l\gets
        \mathbf V^l+\mathbf I_0^l/T+\mathbf I^l(t_{\mathrm s})$.
        \STATE $t_{\mathrm s}\gets t_{\mathrm s}+1$.
    \ENDWHILE
    \STATE
    $\mathbf s^l(t)\gets
    \mathbf1[\mathbf V^l\geq\boldsymbol\theta^l]$.
    \STATE
    $\mathbf V^l\gets
    \mathbf V^l-\boldsymbol\theta^l\odot\mathbf s^l(t)$.
    \STATE Register $\mathbf s^l(t)$ as output slot $t$.
\ENDFOR
\STATE \textbf{return}
$\{\mathbf s^l(t)\}_{t=0}^{T-1}$.
\end{algorithmic}
\end{algorithm}

For network structure design,
Falcon consists of an input spike encoder, a pipelined stem,
and repeated Transformer blocks.
Similar to SpikCommander and SpikeSCR~\citep{wang2026spikcommander,wang2024efficient},
we use a convolutional layer to generate input spikes.
The encoder applies Conv--BN--spike encoding, to produce the input stream $\mathbf S_0$.
The stem contains two convolutional and two linear IF stages:
\begin{equation}
\begin{aligned}
\mathbf S_i
&=\mathrm{SN}_{\delta_{si}}\!\left(
\mathrm{BN}_i(\mathrm{Conv}_i(\mathbf S_{i-1}))\right),
\quad i=1,2,\\
\mathbf X
&=\mathrm{SN}_{\delta_{s4}}\!\left(
\mathrm{SN}_{\delta_{s3}}\!\left(
\operatorname{Reshape}(\mathbf S_2)\mathbf W_P
\right)\mathbf W_E\right).
\end{aligned}
\label{eq:Falcon_stem}
\end{equation}
Here, $\mathrm{SN}_{\delta}$ denotes the IF kernel with firing delay
$\delta$. Operations on spike streams include the fixed scaling and
static contributions defined above.

Each Transformer block computes Q/K/V projections in parallel.
Q and K are quantized after accumulating the complete input window,
while Value uses full-lookahead IF, $\mathrm{SN}_T$.
We replace softmax with
ConSmax~\citep{liu2024consmaxhardwarefriendlyalternativesoftmax},
a hardware-friendly alternative that avoids maximum and sum reductions,
allowing element-wise pipelined evaluation.
For each head $h$, attention computes
\begin{equation}
\mathbf A_h
=\mathcal Q_A\!\left(
\operatorname{ConSmax}\!\left(
\frac{\mathbf Q_h\mathbf K_h^\top}{\sqrt{d_h}}
\right)\right),
\qquad
\mathbf C
=\mathrm{SN}_T\!\left(
\operatorname{Concat}_h
(\mathbf A_h\widehat{\mathbf V}_h)\right).
\label{eq:Falcon_attention}
\end{equation}
where $d_h$ is the head dimension and $\widehat{\mathbf V}_h$
is decoded from the complete Value spike count.
The attention quantizer $\mathcal Q_A$ uses one scale shared across
all heads.
Each AV product starts when both operands are ready, and the resulting
context passes through full-lookahead IF.
The attention output projection and the two FFN layers each use
the delay-controlled IF kernel.
Residual connections after attention and the FFN combine scale-adjusted
inputs from matching logical slots, followed by BN and IF.
After the last block, spike counts are decoded, normalized, pooled,
and passed to the classifier.

Static-weight convolutions and linear projections in the stem and
Transformer blocks are mapped to analog CIM.
After ADC conversion, digital units accumulate partial sums, apply
fixed scaling, and perform IF updates, residual additions,
and buffering.
The input-dependent products $\mathbf Q\mathbf K^\top$ and
$\mathbf A\mathbf V$ are executed digitally along with ConSmax.

\subsection{Theory for Delay in SNNs}
\label{sec:delay_theory}

We study how firing delay affects local spike counts and network
accuracy.
For a fixed input-current stream, a larger delay cannot increase
the count error relative to Full-lookahead.
However, changes in spike timing can still reduce task accuracy,
even when the final counts of the changed layer remain correct.

For the delay-controlled IF kernel in Algorithm~\ref{alg:rate_if}, let
$s_j^l(t;\delta)$ denote the spike produced at output slot $t$
when layer $l$ uses delay $\delta$.
Define the final count and its error relative to Full-lookahead as
\begin{equation}
\widehat q_j^l(\delta)
=
\sum_{t=0}^{T-1}s_j^l(t;\delta),
\qquad
e_j^l(\delta)
=
\left|
\widehat q_j^l(\delta)-\widehat q_j^l(T)
\right|.
\label{eq:local_count_error}
\end{equation}
The reference $\widehat q_j^l(T)$ uses the same input currents,
static term, threshold, and initial membrane potential.

\begin{theorem}
\label{thm:local_count}
Fix the input-current stream and neuron parameters.
For any $1\leq\delta<\delta'\leq T$, the delay-controlled IF kernel
satisfies
\begin{equation}
0
\leq
e_j^l(\delta)-e_j^l(\delta')
\leq
\delta'-\delta,
\qquad
e_j^l(\delta)\leq T-\delta.
\label{eq:count_error_monotonicity}
\end{equation}
\end{theorem}

However, a delay change can preserve the final count while changing the output slots in which spikes occur.
Downstream neurons receive these spikes in a different order,
which can change their firing decisions.
All other delays, weights, and neuron parameters remain fixed.
Let $\operatorname{Acc}_{\mathcal D}(\delta)$ denote its accuracy on a fixed labeled dataset $\mathcal D$.

\begin{proposition}
\label{prop:accuracy_nonmonotone}
There exist a fixed SNN, a fixed finite dataset $\mathcal D$,
and delays $1\leq\delta_1<\delta_2<\delta_3\leq T$ such that
\[
e_j^l(\delta;x)=0
\]
for every neuron $j$ in the changed layer, every
$(x,y)\in\mathcal D$, and every $\delta\in\{1,\ldots,T\}$,
but
\begin{equation}
\operatorname{Acc}_{\mathcal D}(\delta_1)
>
\operatorname{Acc}_{\mathcal D}(\delta_2)
<
\operatorname{Acc}_{\mathcal D}(\delta_3).
\label{eq:accuracy_nonmonotonicity}
\end{equation}
\end{proposition}

Proof sketch:
For Theorem~\ref{thm:local_count}, increasing $\delta$ by one replaces
the earliest partial-input decision with a full-input decision at
the end. The remaining decisions use the same input prefixes.
This leaves the final count unchanged or moves it one spike toward
the full-lookahead count. Repeating this step gives both bounds.
For Proposition~\ref{prop:accuracy_nonmonotone}, we construct spike
streams whose counts remain unchanged but whose spike positions
shift with delay. Signed inputs to downstream neurons then cancel
at different delays, changing their output counts. A fixed readout
produces a correct, incorrect, and then correct prediction as delay
increases.
Full proofs are provided in
Appendices~\ref{app:proof_local_count}
and~\ref{app:proof_accuracy_nonmonotone}.

\subsection{Pipeline Delay Search}
\label{sec:pds}

A larger delay can reduce local count error but still lower network accuracy.
PDS therefore uses two greedy search paths to select layer-wise delay increases based on validation accuracy and modeled core latency within the QNN latency budget.
Let $\mathcal B$ contain the IF stages after stem operations,
attention output projections, FFN layers, and residual additions.
PDS adjusts their delays $\boldsymbol\delta=(\delta^l)_{l\in\mathcal B}$
while keeping model parameters and source encoding fixed.
Q/K use the complete input window, and Value and Context keep
Full-lookahead.
The search targets higher accuracy within the matched QNN latency
budget $L_{\mathrm Q}$.

For a change from $\boldsymbol\delta$ to $\boldsymbol\delta'$, let
$R$ count validation predictions changed from wrong to correct,
and $H$ count those changed from correct to wrong.
We compute the accuracy gain and rescue score as
\begin{equation}
\Delta A=\frac{R-H}{|\mathcal D_{\mathrm{val}}|},
\qquad
S_{\mathrm{rescue}}=\frac{R-H}{\sqrt{R+H}},
\label{eq:pds_score}
\end{equation}
with $S_{\mathrm{rescue}}=0$ when $R+H=0$.

A change is accepted here if:
\begin{equation}
S_{\mathrm{rescue}}\geq\tau,
\qquad
\Delta A\geq\epsilon_{\mathrm{acc}},
\qquad
L_{\mathrm{DAG}}(\boldsymbol\delta')<L_{\mathrm Q}.
\label{eq:pds_filter}
\end{equation}
We use $\tau=2$ and $\epsilon_{\mathrm{acc}}=0.001$,
corresponding to a minimum gain of $0.1$ percentage points.

The change with the largest accuracy gain may also increase latency.
PDS uses two search paths: one selects the largest accuracy gain,
and the other selects the largest gain per added latency.
Let $L(\boldsymbol\delta)$ be the network latency under schedule
$\boldsymbol\delta$, and let
$\Delta L=L(\boldsymbol\delta')-L(\boldsymbol\delta)$.
The two scores are
\begin{equation}
U_{\mathrm{acc}}=\Delta A,
\qquad
U_{\mathrm{eff}}=
\begin{cases}
\Delta A/\Delta L, & \Delta L>0,\\
+\infty, & \Delta L\leq0.
\end{cases}
\label{eq:pds_utility}
\end{equation}

PDS starts with delay one at every searchable IF stage.
If candidate scores are equal, we choose the larger accuracy gain,
then the smaller added latency.
In our setting, algorithm~\ref{alg:pds} searches delays up to $\delta_{\max}=3$ and updates candidate scores after each selected change.

\begin{algorithm}[ht]
\caption{Pipeline Delay Search (PDS)}
\label{alg:pds}
\begin{algorithmic}[1]
\REQUIRE
Fixed model;
$\mathcal D_{\mathrm{val}},\mathcal B,\delta_{\max}$;
latency function $L$, budget $L_{\mathrm Q}$;
$\tau,\epsilon_{\mathrm{acc}}$.
\ENSURE
Fastest, Balanced, and Accurate schedules.

\STATE Initialize $\boldsymbol\delta^{(0)}=(1,\ldots,1)$,
with one delay for each stage in $\mathcal B$.
\STATE $\mathcal S\gets\{\boldsymbol\delta^{(0)}\}$.
\FOR{$m\in\{\mathrm{acc},\mathrm{eff}\}$}
    \STATE $\boldsymbol\delta\gets\boldsymbol\delta^{(0)}$.
    \WHILE{\textbf{true}}
        \STATE Form candidates $\mathcal E$ by increasing one
        $\delta^l$ to each
        $d\in\{\delta^l+1,\ldots,\delta_{\max}\}$,
        for every $l\in\mathcal B$.
        \STATE Compare candidates with $\boldsymbol\delta$
        and keep those passing Eq.~\ref{eq:pds_filter}.
        \IF{$\mathcal E=\varnothing$}
            \STATE \textbf{break}
        \ENDIF
        \STATE Choose $\boldsymbol\delta'\in\mathcal E$
        with the largest $U_m$.
        \STATE $\boldsymbol\delta\gets\boldsymbol\delta'$.
        \STATE $\mathcal S\gets\mathcal S\cup\{\boldsymbol\delta\}$.
    \ENDWHILE
\ENDFOR
\STATE Remove a schedule from $\mathcal S$ if another is at least
as accurate and at least as fast, with a strict improvement in one.
\STATE \textbf{return} Fastest, Balanced, and Accurate from $\mathcal S$.
\end{algorithmic}
\end{algorithm}

From the remaining schedules in $\mathcal S$, \emph{Fastest} has
the lowest latency and \emph{Accurate} has the highest accuracy.
To choose \emph{Balanced}, we scale accuracy and latency separately
to $[0,1]$ using their minimum and maximum values in $\mathcal S$.
We then subtract scaled latency from scaled accuracy and select
the schedule with the highest score.

\section{Results}
\label{sec:results}

We evaluate Falcon on the 35-class Google Speech Commands v2
(GSC) and Spiking Speech Commands (SSC) benchmarks
\citep{warden2018speech,cramer2022heidelberg}. We use two model sizes, Falcon-Medium and Falcon-Large, with three
and five Transformer blocks, respectively. Additional experimental details, including data preprocessing and
model architectures, are provided in the Appendix.

Analog costs in Falcon are obtained from NeuroX, calibrated against
measurements from a 1-Mb ReRAM CIM macro~\cite{zhu_neurox, 8910375}.
Digital latency is computed using cycle-based models at 100~MHz.
Digital energy is evaluated separately using Design Compiler
synthesis in a commercial 22-nm process at the TT corner,
0.65~V, and $25\,^{\circ}\mathrm{C}$. Memory energy uses low-power memory-compiler configurations at the same corner.
We combine operator times through the dependency graph rather
than infer latency from operation counts alone.
All reported latencies are modeled network-core latencies, from
available source activations to the final Transformer block. 
Training runs on Ubuntu 24.04.4 LTS with two AMD EPYC 9355 CPUs,
1~TiB of RAM, and four NVIDIA RTX PRO 6000 Blackwell Server Edition
GPUs with 96~GB each. We first train a QNN and transfer its weights to an SNN for inference, following prior QNN-to-SNN conversion methods~\citep{3780338.3782679,yan2026otters}.
The converted SNN serves as the starting point for further training.

\subsection{More lookahead does not guarantee higher accuracy}
\label{sec:results_lookahead}

Figure~\ref{fig:delay_effects}(a--d) summarizes our delay
ablations on Falcon-Large using GSC, with weights and neuron
parameters fixed.
In Figure~\ref{fig:delay_effects}(a), increasing the uniform
delay from $\delta=1$ to $2$ raises modeled core latency
from $180.42$ to $231.24\,\mu\mathrm{s}$, and also lowers validation
accuracy from $11.64\%$ to $4.28\%$.
More waiting can make inference both slower and less accurate.
Figure~\ref{fig:delay_effects}(b) further examines the effect
of increasing delay at each boundary.
As Figure~\ref{fig:delay_effects}(a) shows, the matched-QNN
latency lies between the SNN latencies with uniform delays
of one and two, leaving limited room for added waiting.
We therefore test single-boundary changes from $\delta=1$
to $2$ or $3$, while keeping all other searchable delays
at one. Figure~\ref{fig:delay_effects}(c) presents three representative
cases.
Using the screening criteria defined in the Section~\ref{sec:pds},
we identify delay increases with sufficient task-level gains
and highlight them in the green region of
Figure~\ref{fig:delay_effects}(d).
These criteria guide the Pipeline Delay Search described next.

\begin{figure}[hbt]
    \centering
    \includegraphics[width=1.05\linewidth]{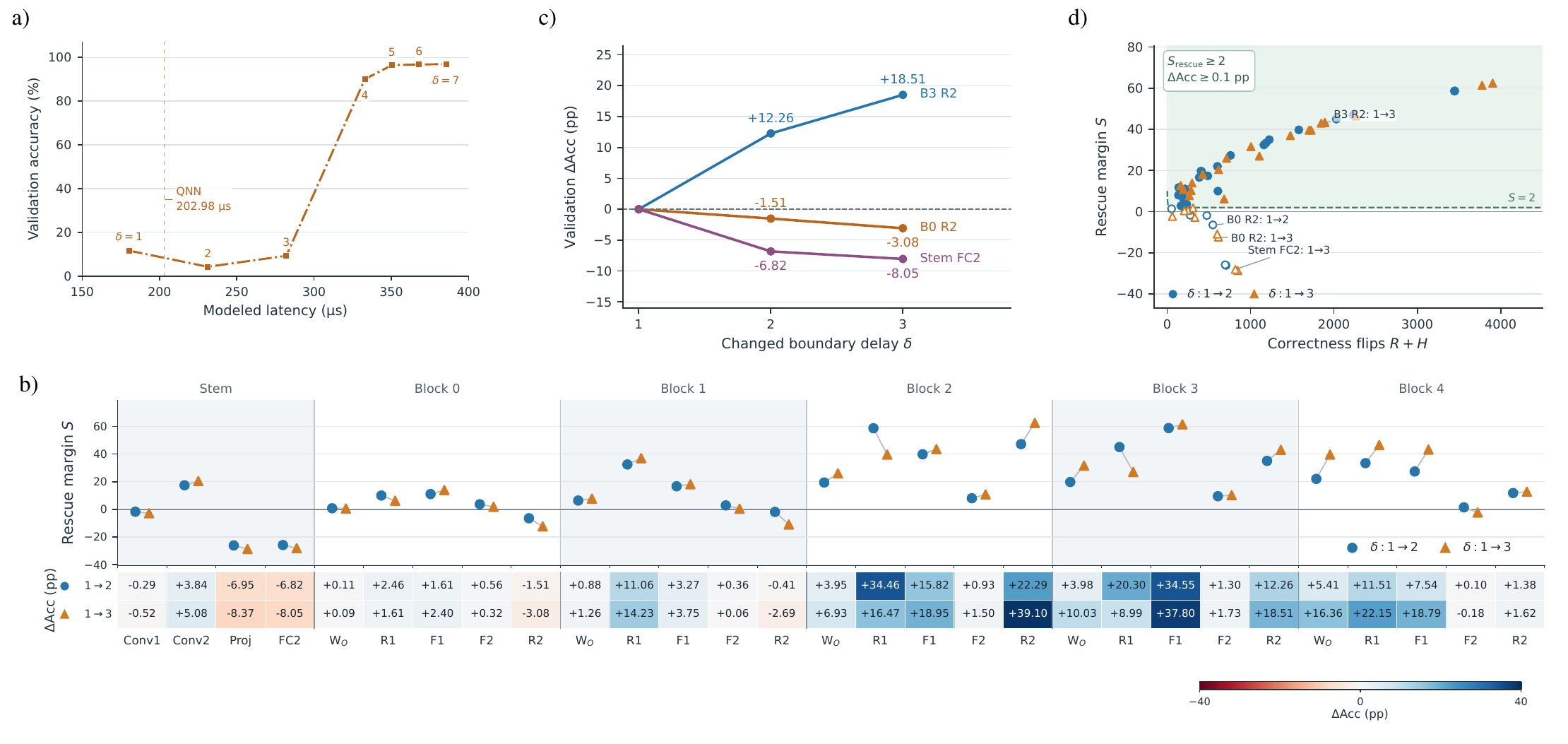}
\caption{Effects of firing delay on Falcon-Large accuracy and latency.(a) Accuracy and latency under uniform delays.
(b) Layerwise accuracy changes for delay increases
$1\!\to\!2$ and $1\!\to\!3$.
(c) Three representative layer responses.
(d) PDS candidate screening, with delay increases passing
both score and accuracy-gain gates highlighted in green.}
    \label{fig:delay_effects}
\end{figure}

\subsection{Pipeline Delay Search and adaptation}
\label{sec:results_pds}

Figure~\ref{fig:pds_adaptation}(a) shows the schedules selected
by PDS before adaptation.
On GSC with Falcon-Large, the two search paths evaluate 916 unique schedules. Fastest, Balanced, and Accurate reach validation accuracies of
$11.64\%$, $73.28\%$, and $91.99\%$ at
$180.42$, $184.24$, and $201.16\,\mu\mathrm{s}$, respectively. The configurations for all three schedules are provided in the Appendix~\ref{app:model_architecture}.
To test whether delay placement matters, we compare each PDS
schedule with ten random delay settings that use the same numbers
of each delay value and have the same core latency.
Figure~\ref{fig:pds_adaptation}(b) reports these frozen-model
results. For example, at the Balanced latency, PDS reaches $73.28\%$ validation accuracy, compared with $23.64\pm12.23\%$ for random settings. Furthermore, even after training, models using PDS-selected schedules
maintain competitive accuracy than those using randomly selected
schedules. Detailed experiments are presented in
Appendix~\ref{app:pds_post_training}.

Figure~\ref{fig:pds_adaptation}(c) shows the training process
after delay search. The SNN is first trained for 30 epochs with
spike-based QAT. The network weights are then fixed, and only
the firing thresholds and initial membrane potentials are tuned
for 15 epochs. Finally, the network weights and neuron parameters
are jointly trained for another 10 epochs. This process increases validation accuracy from $11.64\%$ to $95.88\%$ for Fastest, from $73.28\%$ to $96.17\%$ for Balanced, and from $91.99\%$ to $96.49\%$ for Accurate.

\begin{figure}[ht]
    \centering
    \includegraphics[width=1.05\linewidth]{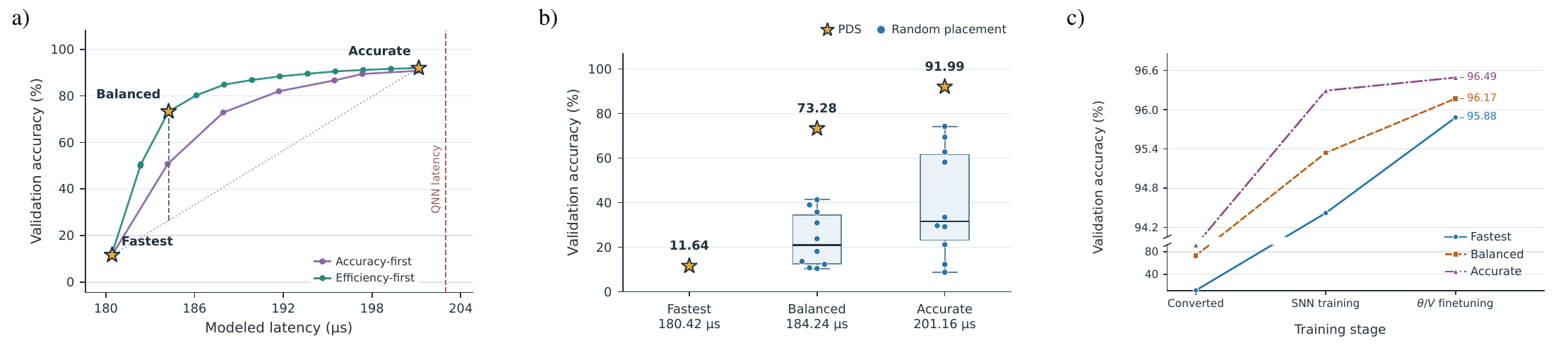}
    \caption{
   Delay search and training of Falcon-Large on GSC.
(a) Validation accuracy and latency of Fastest, Balanced,
and Accurate before SNN training.
(b) Comparison with ten random delay schedules at the same
core latency before SNN training.
(c) Validation accuracy at each SNN training stage.
    }
    \label{fig:pds_adaptation}
\end{figure}

\subsection{Final accuracy and latency}
\label{sec:results_final}

Table~\ref{tab:pds_main_results} reports the final test accuracy
and modeled core latency of FALCON-Medium on SSC and GSC,
and Falcon-Large on GSC.
The \textit{Fastest} setting uses $\delta=1$ at all searchable
IF stages.
On GSC, FALCON-Medium in this setting achieves $96.11\%$
accuracy at $115.82\,\mu\mathrm{s}$, compared with $96.68\%$
at $130.86\,\mu\mathrm{s}$ for its matched QNN.
This provides a $1.13\times$ speedup with a
$0.57$-percentage-point accuracy loss.
The \textit{Accurate} setting reduces this gap to $0.29$
percentage points at $129.04\,\mu\mathrm{s}$.
On SSC, FALCON-Medium with the \textit{Balanced} schedule
achieves $83.02\%$ accuracy at $124.00\,\mu\mathrm{s}$.
Compared with \textit{Fastest}, it gains $1.78$ percentage
points with only $7.58\,\mu\mathrm{s}$ of additional latency,
offering a favorable accuracy--latency trade-off while remaining
faster than the matched QNN ($83.97\%$ at
$131.46\,\mu\mathrm{s}$).
Full-lookahead configurations achieve higher accuracy
at the cost of longer latency.
\begin{table}[hbt]
\centering
\begin{threeparttable}
\caption{Test accuracy and modeled network-core latency.}
\label{tab:pds_main_results}
\small
\setlength{\tabcolsep}{3pt}
\renewcommand{\arraystretch}{0.8}
\begin{tabular*}{\textwidth}{@{\extracolsep{\fill}}lrrrrrr@{}}
\toprule
& \multicolumn{2}{c}{\textbf{SSC Falcon-Medium}}
& \multicolumn{2}{c}{\textbf{GSC Falcon-Medium}}
& \multicolumn{2}{c}{\textbf{GSC Falcon-Large}} \\
\cmidrule(lr){2-3}\cmidrule(lr){4-5}\cmidrule(lr){6-7}
Configuration
& \shortstack{Acc.(\%)} & \shortstack{latency($\mu$s)}
& \shortstack{Acc.(\%)} & \shortstack{latency($\mu$s)}
& \shortstack{Acc.(\%)} & \shortstack{latency($\mu$s)} \\
\midrule
Matched QNN
& 83.97 & 131.46
& 96.68 & 130.86
& 96.92 & 202.98 \\
Full-lookahead
& 83.94 & 253.04
& 96.68 & 252.44
& 96.91 & 385.44 \\
\midrule
PDS--Fastest
& 81.24 & \textbf{116.42}
& 96.11 & \textbf{115.82}
& 95.46 & \textbf{180.42} \\
PDS--Balanced
& 83.02 & 124.00
& 96.31 & 119.64
& 96.19 & 184.24 \\
PDS--Accurate
& 83.02 & 129.70
& 96.39 & 129.04
& 96.52 & 201.16 \\
\bottomrule
\end{tabular*}
\end{threeparttable}
\end{table}

Tables~\ref{tab:gsc_latency_comparison}
and~\ref{tab:gsc_ssc_comparison} place these results alongside
reported SNN speech-recognition results.
On GSC, FALCON-Medium with the \textit{Balanced} schedule achieves
$96.31\%$ accuracy with $0.82$M parameters, exceeding the
$96.08\%$ reported by SpikeSCR while using $73.9\%$
fewer parameters.
Relative to their non-pipelined counterparts, the Medium and
Large \textit{Balanced} configurations reduce modeled core latency
by $52.6\%$ and $52.2\%$, respectively, with accuracy losses
of $0.37$ and $0.72$ percentage points.
These results show that the \textit{Balanced} configurations
achieve competitive recognition accuracy while more than halving
modeled core latency.

\begin{table}[hbt]
\centering
\caption{
Comparison on GSC in terms of model size, accuracy, and latency.
As most prior works do not report hardware latency, we reproduce the latency of several representative state-of-the-art SNN models with digital arrays of the same size of Falcon. Detailed implementations and latency reconstruction are provided in Appendix~\ref{app:latency_reconstruction}.
For Falcon, results are reported as \textit{Non-pipelined / Balanced}.
}
\label{tab:gsc_latency_comparison}
\small
\setlength{\tabcolsep}{5pt}
\renewcommand{\arraystretch}{0.8}
\begin{tabular}{lccc}
\toprule
Method & Params. (M) & Accuracy (\%) & Latency ($\mu$s) \\
\midrule

T-BSO~\cite{pmlr-v267-liang25r}, Spiking VGG-11
& 9.23
& 96.12
& 1714.88 \\

\midrule

SpikeSCR~\cite{wang2024efficient}, 1L-16-256 (short)
& 1.63
& 94.71
& 278.47 \\

SpikeSCR, 1L-16-256
& 1.63
& 95.90
& 468.11 \\

\midrule

SpikCommander~\cite{wang2026spikcommander}, 1L-16-256
& 1.12
& 96.71
& 487.97 \\

SpikCommander, 2L-16-256 (short)
& 2.13
& 96.27
& 645.63 \\

SpikCommander, 2L-16-256
& 2.13
& 96.92
& 940.70 \\

\midrule

\textbf{Falcon-Medium}
& \textbf{0.82}
& $\mathbf{96.68}\,/\,\mathbf{96.31}$
& $\mathbf{252.44}\,/\,\mathbf{119.64}$ \\

\textbf{Falcon, Large}
& \textbf{1.24}
& $\mathbf{96.91}\,/\,\mathbf{96.19}$
& $\mathbf{385.44}\,/\,\mathbf{184.24}$ \\

\bottomrule
\end{tabular}
\end{table}

\begin{table}[hbt]
\centering
\caption{Comparison with prior SNN speech-recognition methods.
Falcon reports \emph{Full-lookahead (direct conversion) / PDS--Accurate}.``--'' denotes an unreported or unverified value.
}
\label{tab:gsc_ssc_comparison}
\scriptsize
\setlength{\tabcolsep}{4pt}
\renewcommand{\arraystretch}{1.05}
\resizebox{\linewidth}{!}{%
\begin{tabular}{@{}lcccc@{}}
\toprule
& \multicolumn{2}{c}{SSC}
& \multicolumn{2}{c}{GSC} \\
\cmidrule(lr){2-3}
\cmidrule(lr){4-5}
Method / Configuration
& Params. (M) & Accuracy (\%)
& Params. (M) & Accuracy (\%) \\
\midrule
DCLS-Delays, 2L-2KC~\citep{hammouamri2024dcls}
& 1.40 & $80.16\pm0.09$
& 1.40 & $95.00\pm0.06$ \\

DCLS-Delays, 3L-2KC
& 2.50 & $80.69\pm0.21$
& 2.50 & $95.35\pm0.04$ \\

d-cAdLIF~\citep{deckers2024dcadlif}
& 0.70 & $80.23\pm0.07$
& 0.61 & $95.69\pm0.03$ \\

SNN-Delays+ TR/NAR~\citep{zhang2024temporalinformationreconstructionnonaligned}
& 2.50 & 81.02
& 2.50 & 95.62 \\

CADAD, 3L~\citep{bai2026cadad}
& 0.60 & $80.69\pm0.24$
& 0.60 & $95.58\pm0.15$ \\

SE-adLIF, 2L~\citep{baronig2025seadlif}
& 1.60 & $80.44\pm0.26$
& -- & -- \\

SpikeSCR + KDCL, 2L-16-256~\citep{wang2024efficient}
& 3.15 & 83.69
& 3.15 & 96.08 \\

SpikCommander, 1L-16-256~\citep{wang2026spikcommander}
& 1.12 & 83.26
& 1.12 & 96.71 \\

SpikCommander, 2L-16-256
& 2.13 & 83.49
& 2.13 & 96.92 \\

SNN-KWS~\citep{wang2024globallocal}
& -- & --
& 0.0802 & 92.90 \\

SIDC-KWS~\citep{lim2025sidc}
& -- & --
& 0.4028 & 94.70 \\

Spiking LMUFormer~\citep{liu2024lmuformer}
& -- & --
& 1.69 & 96.12 \\
\midrule
\textbf{Falcon-Medium}
& 0.97 & $\textbf{83.94}\,/\,83.02$
& 0.82 & $96.68\,/\,96.39$ \\

\textbf{Falcon-Large}
& -- & --
& 1.24 & $\textbf{96.91}\,/\,96.52$ \\
\bottomrule
\end{tabular}%
}
\end{table}

\section{Discussion and conclusion}
\subsection{Energy and throughput discussion}

Prior SNN studies commonly estimate energy by multiplying
operation counts by fixed per-operation energy
costs~\citep{wang2024efficient,wang2026spikcommander}.
Following this convention, we first estimate Falcon's arithmetic
energy under a hypothetical all-digital implementation.
We average operation counts over the complete GSC test set,
including the input encoder and final classifier, and use
precision-specific energy costs obtained from 22-nm RTL
synthesis and simulation.
Under this model, Falcon-Medium and Falcon-Large with the
\textit{Balanced} schedule require an estimated $0.01405$ and
$0.02135\,\mathrm{mJ}$ per inference, respectively.
For reference, SpikCommander reports $0.028$ and
$0.042\,\mathrm{mJ}$ for its one- and two-block
models~\citep{wang2026spikcommander}.

However, operation-count estimates cover only arithmetic energy,
excluding preprocessing, memory access, communication, and
operations such as exponentiation and comparison.
We therefore additionally evaluate data-dependent energy.
For the analog part, we use NeuroX with hardware parameters
matched to the 1-Mb ReRAM-based CIM macro~\cite{zhu_neurox,8910375}.
Digital latency is computed using cycle-based models at
$100\,\mathrm{MHz}$. Digital energy settings and activity sources are detailed in Appendix~\ref{sec: energybreak}.
Combining the analog and digital contributions yields modeled core
energies of $0.75$ and $1.06\,\mathrm{mJ}$ per inference for
Falcon-Medium and Falcon-Large, respectively, under the
\textit{Balanced} schedule. We show the detailed breakdown in Appendix~\ref{sec: energybreak}.

Throughput measures the number of samples completed per unit time
during continuous inference.
By reducing idle time between layers, cross-timestep pipelining
may allow the network to complete samples more frequently.
Whether this benefit is realized depends on resource sharing,
buffering, and the selected firing delays.
In particular, increasing $\delta$ may add inter-layer waiting,
but this reduces throughput only if it increases the average
time between consecutive sample completions.
Thus, completing one sample earlier does not necessarily mean
completing more samples per unit time.
This work focuses on the former: reducing modeled single-sample
core latency rather than improving throughput.

\subsection{Conclusion}
We introduced FALCON, which enables cross-timestep pipelining
through a delay-controlled IF kernel.
By overlapping computation across layers, SNNs can finish
earlier despite requiring more local processing rounds.
We also showed that longer waiting does not always improve
network accuracy.
Pipeline Delay Search therefore selects where and how long
to wait, followed by training under the selected delays.
Experiments on GSC and SSC demonstrate competitive recognition
accuracy.
Full-lookahead FALCON achieves competitive accuracy among
the compared SNNs, while the \textit{Balanced} schedules reduce
modeled core latency with small accuracy losses.
Overall, under the evaluated mapping, FALCON achieves lower modeled core
latency than matched bit-serial QNNs and provides a practical way to balance accuracy and latency through layer-wise firing delays.

\newpage
\section*{AI Use Statement}
The authors acknowledge the use of generative AI tools, including ChatGPT, during the preparation of this manuscript and the research process. These tools were used to assist with improving the readability and clarity of the manuscript, as well as to support preliminary drafting, code development, and research-related discussions. All AI-assisted content, code, and suggestions were carefully reviewed, revised, and validated by the authors. The authors take full responsibility for the accuracy, originality, and integrity of the final manuscript.
\bibliography{iclr2027_conference}
\bibliographystyle{iclr2027_conference}

\appendix
\clearpage
\raggedbottom
\setlength{\intextsep}{8pt plus 2pt minus 2pt}
\captionsetup{skip=6pt}
\section{Ablation study}

\subsection{PDS for Falcon-Medium on GSC and SSC }
\begin{figure}[H]
    \centering
    \includegraphics[width=\linewidth]{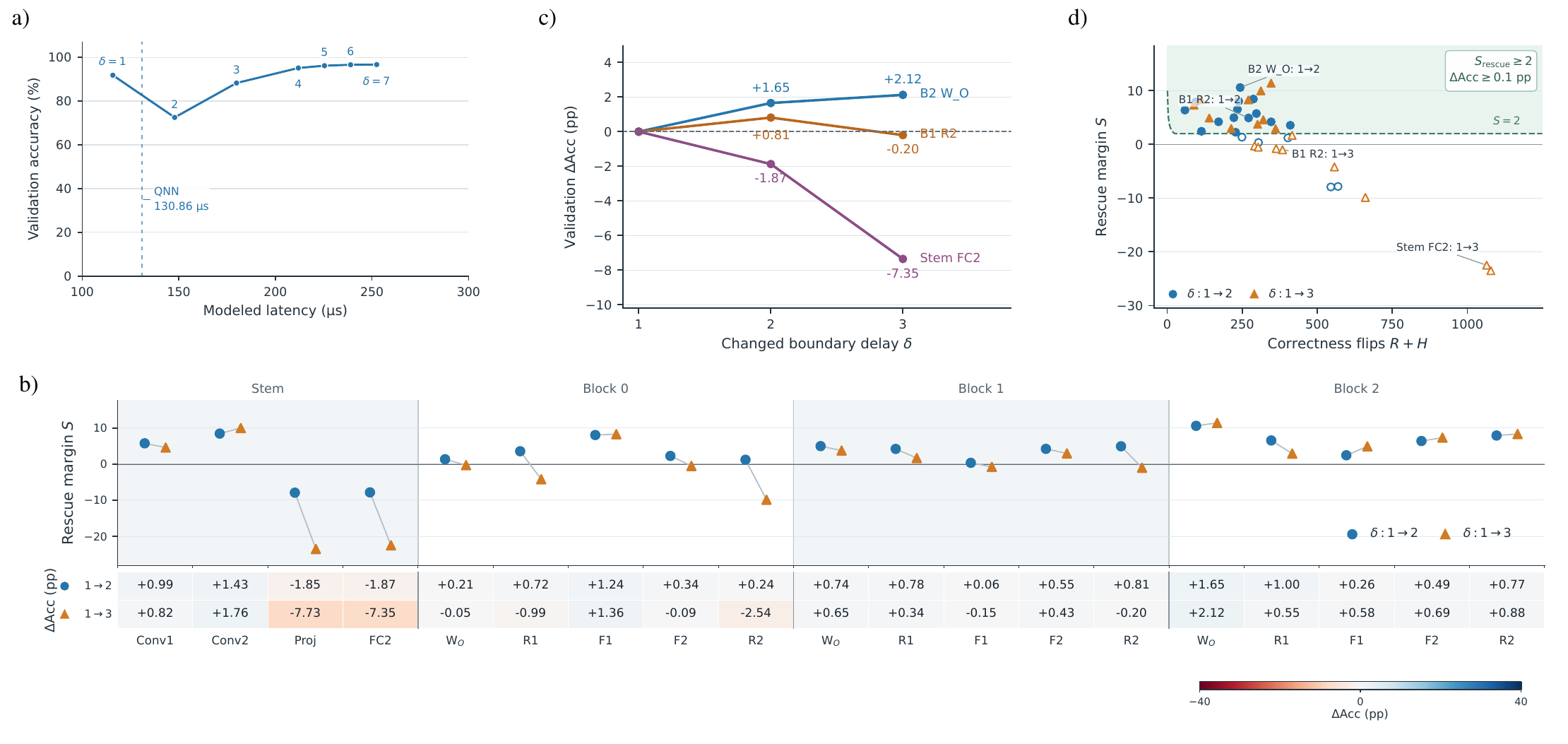}
\caption{Effects of firing delay on Falcon-Medium accuracy and latency (GSC).(a) Accuracy and latency under uniform delays.
(b) Layerwise accuracy changes for delay increases
$1\!\to\!2$ and $1\!\to\!3$.
(c) Three representative layer responses.
(d) PDS candidate screening, with delay increases passing
both score and accuracy-gain gates highlighted in green.}
    \label{fig:delay_effects_1}
\end{figure}

\begin{figure}[H]
    \centering
    \includegraphics[width=\linewidth]{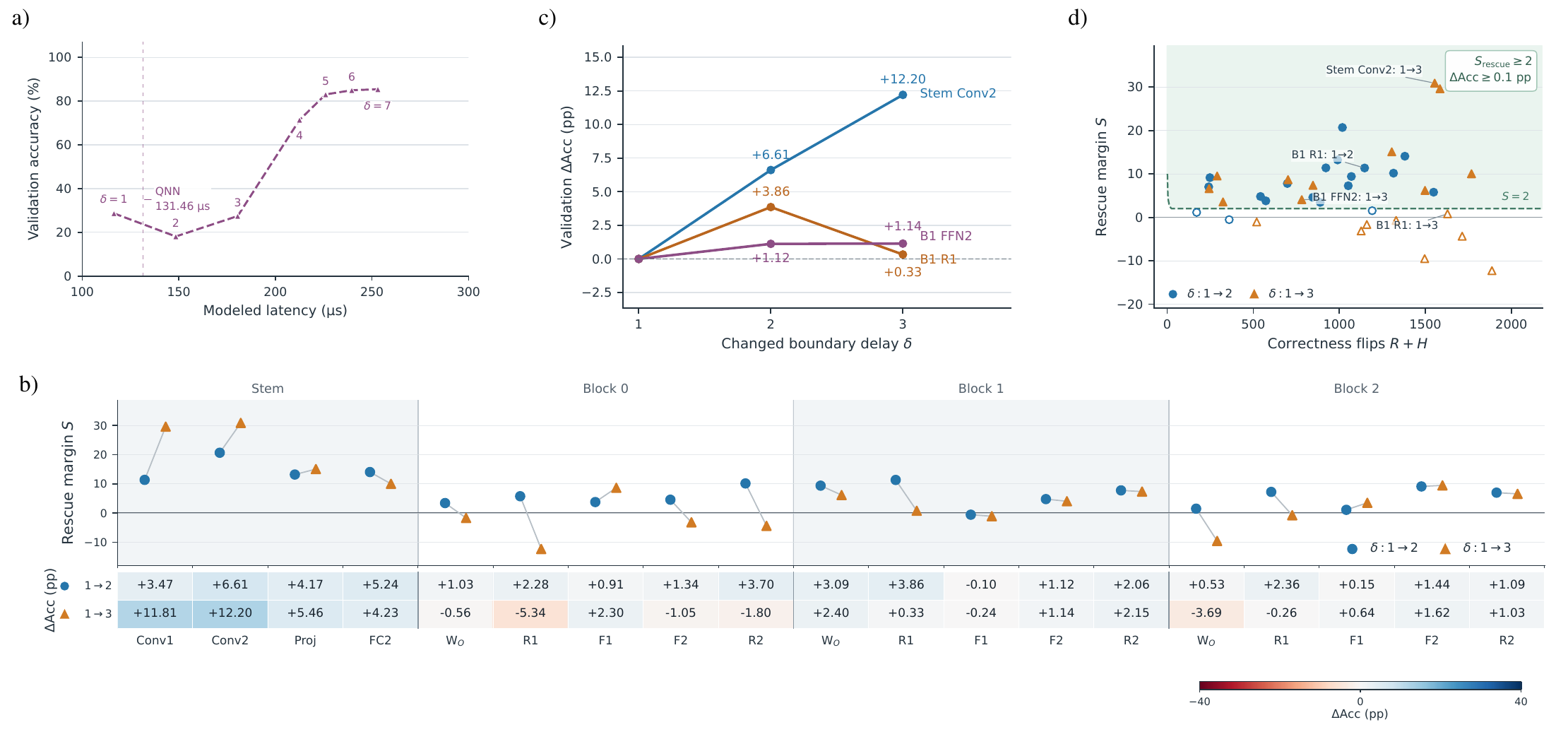}
\caption{Effects of firing delay on Falcon-Medium accuracy and latency (SSC).(a) Accuracy and latency under uniform delays.
(b) Layerwise accuracy changes for delay increases
$1\!\to\!2$ and $1\!\to\!3$.
(c) Three representative layer responses.
(d) PDS candidate screening, with delay increases passing
both score and accuracy-gain gates highlighted in green.}
    \label{fig:delay_effects_1_}
\end{figure}

\begin{figure}[H]
    \centering
    \includegraphics[width=\linewidth]{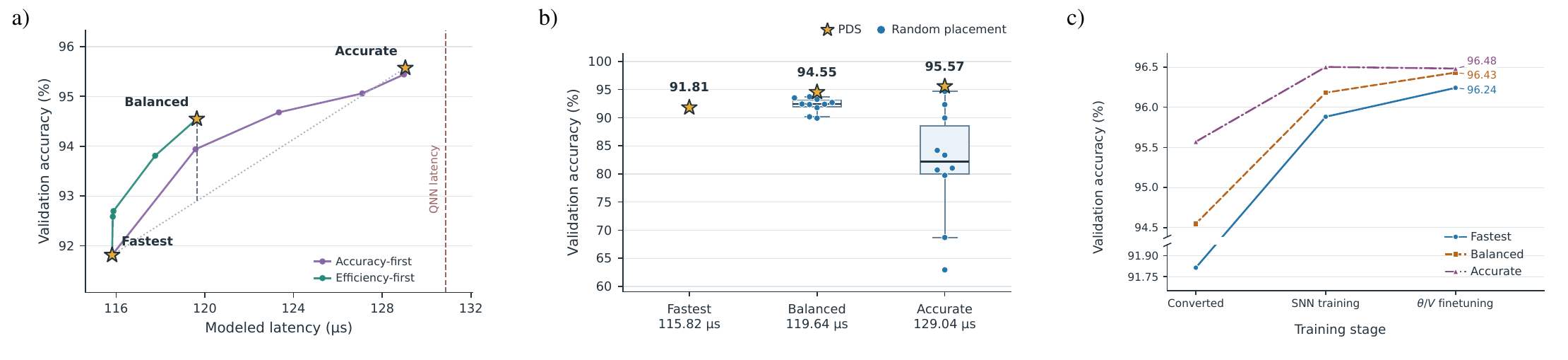}
    \caption{
   Delay search and training of Falcon-Medium on GSC.
(a) Validation accuracy and latency of Fastest, Balanced,
and Accurate before SNN training.
(b) Comparison with ten random delay schedules at the same
core latency before SNN training.
(c) Validation accuracy at each SNN training stage.
    }
    \label{fig:pds_adaptation_2}
\end{figure}

\begin{figure}[H]
    \centering
    \includegraphics[width=\linewidth]{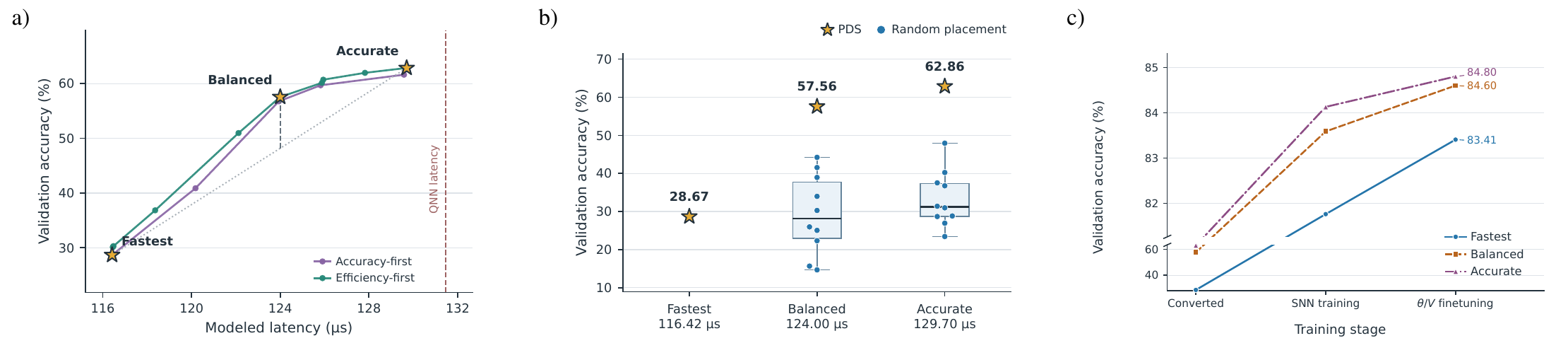}
    \caption{
   Delay search and training of Falcon-Medium on SSC.
(a) Validation accuracy and latency of Fastest, Balanced,
and Accurate before SNN training.
(b) Comparison with ten random delay schedules at the same
core latency before SNN training.
(c) Validation accuracy at each SNN training stage.
    }
    \label{fig:pds_adaptation_2_}
\end{figure}

\subsection{PDS schedules and QAT recovery}
\label{app:pds_post_training}

Figure~\ref{fig:qat30_validation_trajectories} shows the
30-epoch QAT validation curves for the PDS-selected
\textit{Balanced} schedules on GSC and SSC.
Most accuracy recovery occurs within the first ten epochs,
motivating a ten-epoch budget for the following comparison.

\begin{figure}[H]
    \centering
    \includegraphics[width=\linewidth]{
        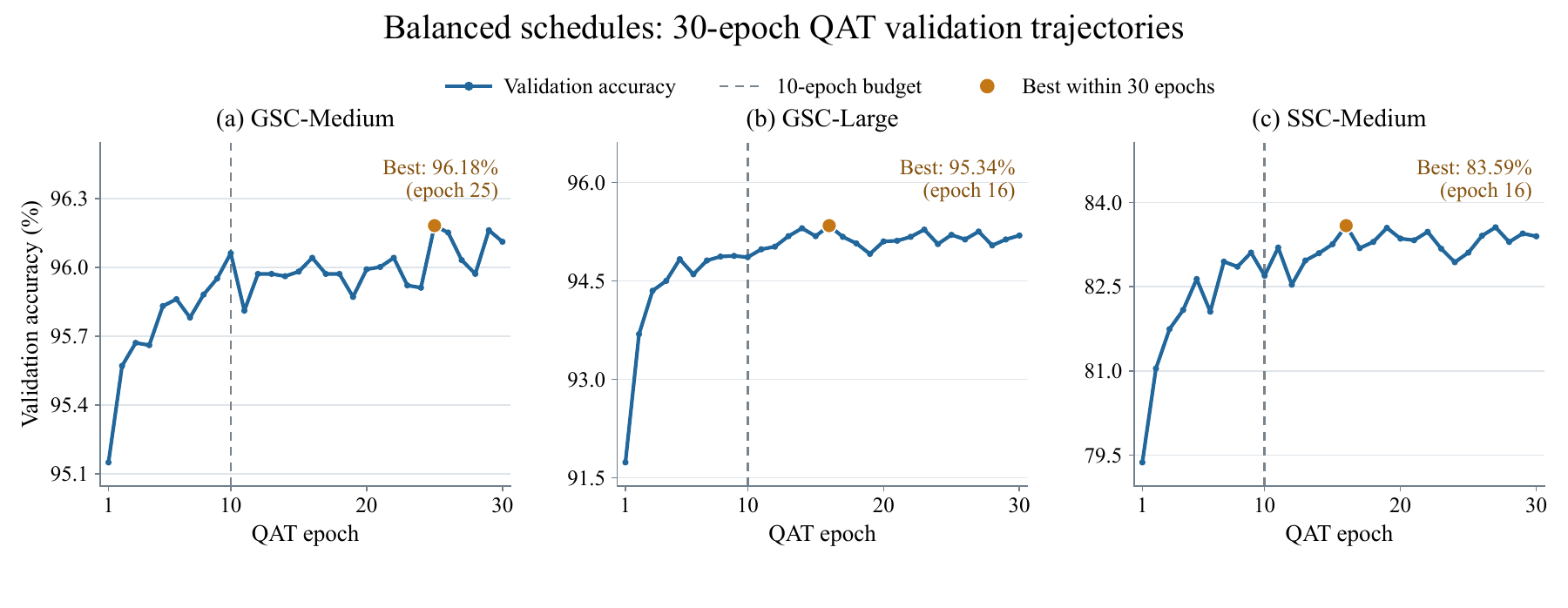
    }
    \caption{
        Validation accuracy during 30 epochs of QAT for
        PDS-selected \textit{Balanced} schedules:
        (a) GSC-Medium, (b) GSC-Large, and (c) SSC-Medium.
        Dashed lines mark the ten-epoch budget, and highlighted
        points indicate the highest validation accuracy
        within 30 epochs.
    }
    \label{fig:qat30_validation_trajectories}
\end{figure}
We compare the PDS-selected \textit{Balanced} schedule with five
random schedules for FALCON-Medium on GSC and SSC, and
Falcon-Large on GSC.
Within each setting, schedules share the same delay distribution,
modeled core latency, initialization, and ten-epoch QAT budget.
PDS achieves the highest validation accuracy at epoch ten
in all three settings.
On SSC-Medium, it also remains highest throughout training.

We further train the same five random schedules using the full
training pipeline: 30 epochs of QAT, 15 epochs of IF-parameter tuning,
and 10 epochs of joint training.
The PDS-selected \textit{Balanced} schedules achieve test accuracies of
$96.31\%$, $96.19\%$, and $83.02\%$ on GSC-Medium, GSC-Large, and
SSC-Medium, exceeding the corresponding random-schedule means by
$0.11$, $0.23$, and $0.92$ percentage points, respectively.
These results show that the benefits of PDS beyond early QAT recovery.

\begin{table}[t]
\centering
\caption{Test accuracy (\%) after full adaptation with
30 epochs of QAT, 15 epochs of neuron-parameter tuning,
and 10 epochs of joint training.
Each random schedule matches the delay distribution and
modeled core latency of the corresponding PDS Balanced schedule.
All runs use the same training seed, and checkpoints are
selected by validation accuracy.}
\label{tab:pds_random_full_adaptation}
\small
\setlength{\tabcolsep}{6pt}
\begin{tabular}{@{}lrrr@{}}
\toprule
Schedule & GSC Falcon-Medium & GSC Falcon-Large & SSC Falcon-Medium \\
\midrule
PDS              & 96.31 & 96.19 & 83.02 \\
Random 1         & 96.08 & 96.18 & 82.08 \\
Random 2         & 96.33 & 96.02 & 82.39 \\
Random 3         & 96.25 & 95.78 & 82.41 \\
Random 4         & 96.24 & 95.81 & 81.66 \\
Random 5         & 96.13 & 96.00 & 82.01 \\
\midrule
Random average   & 96.21 & 95.96 & 82.11 \\
PDS gain (pp)    & +0.11 & +0.23 & +0.92 \\
\bottomrule
\end{tabular}
\end{table}

\begin{figure}[H]
    \centering
    \includegraphics[width=\linewidth]{
        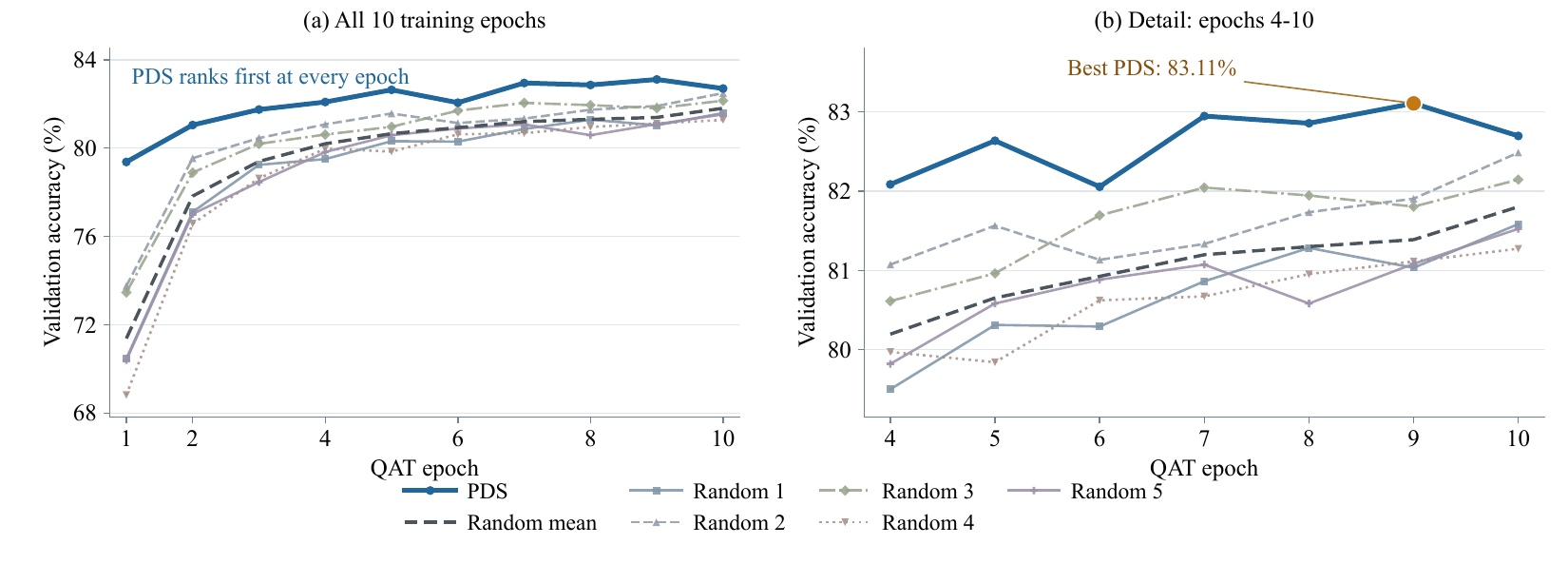
    }
    \caption{
        Equal-budget QAT comparison on SSC with FALCON-Medium.
        The PDS-selected \textit{Balanced} schedule and five
        random schedules share the same delay distribution
        and modeled core latency.
        (a) Validation accuracy over all ten training epochs.
        (b) A closer view of epochs 4--10.
    }
    \label{fig:ssc_pds_random_qat}
\end{figure}

\begin{figure}[H]
    \centering
    \includegraphics[width=\linewidth]{
        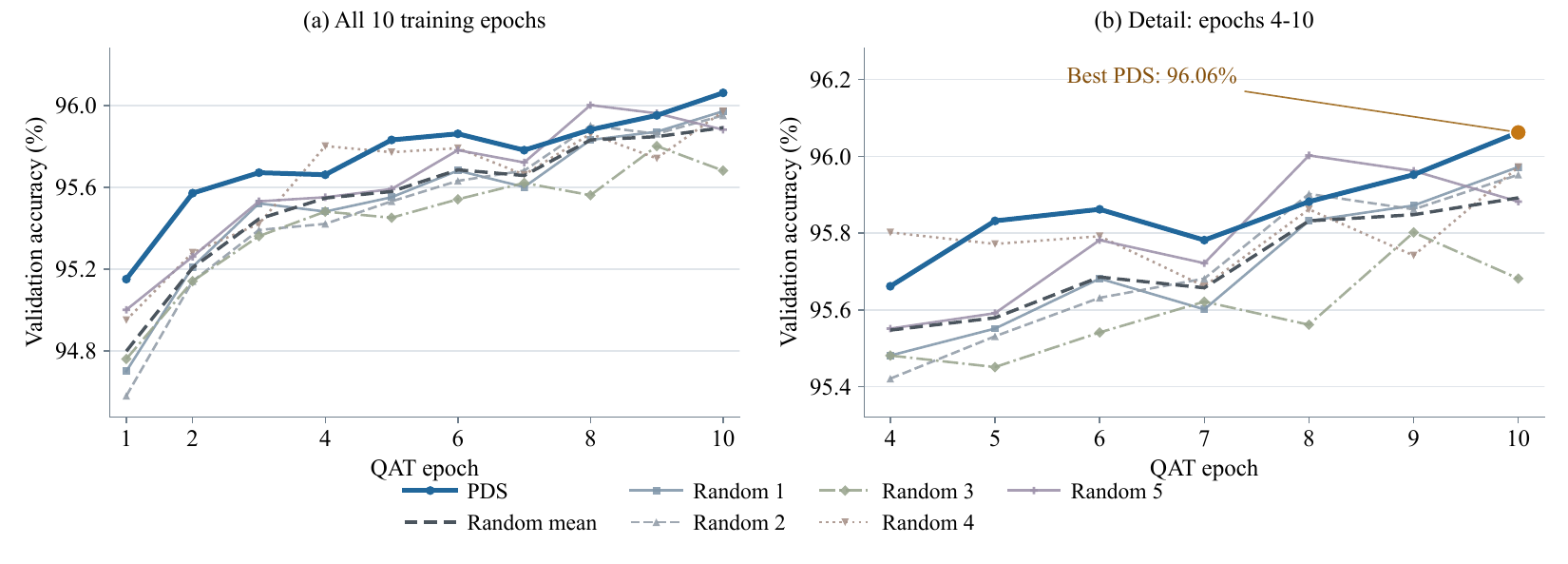
    }
    \caption{
        Equal-budget QAT comparison on GSC with FALCON-Medium.
        The PDS-selected \textit{Balanced} schedule and five
        random schedules share the same delay distribution
        and modeled core latency.
        (a) Validation accuracy over all ten training epochs.
        (b) A closer view of epochs 4--10.
    }
    \label{fig:gsc_m__pds_random_qat}
\end{figure}

\begin{figure}[H]
    \centering
    \includegraphics[width=\linewidth]{
        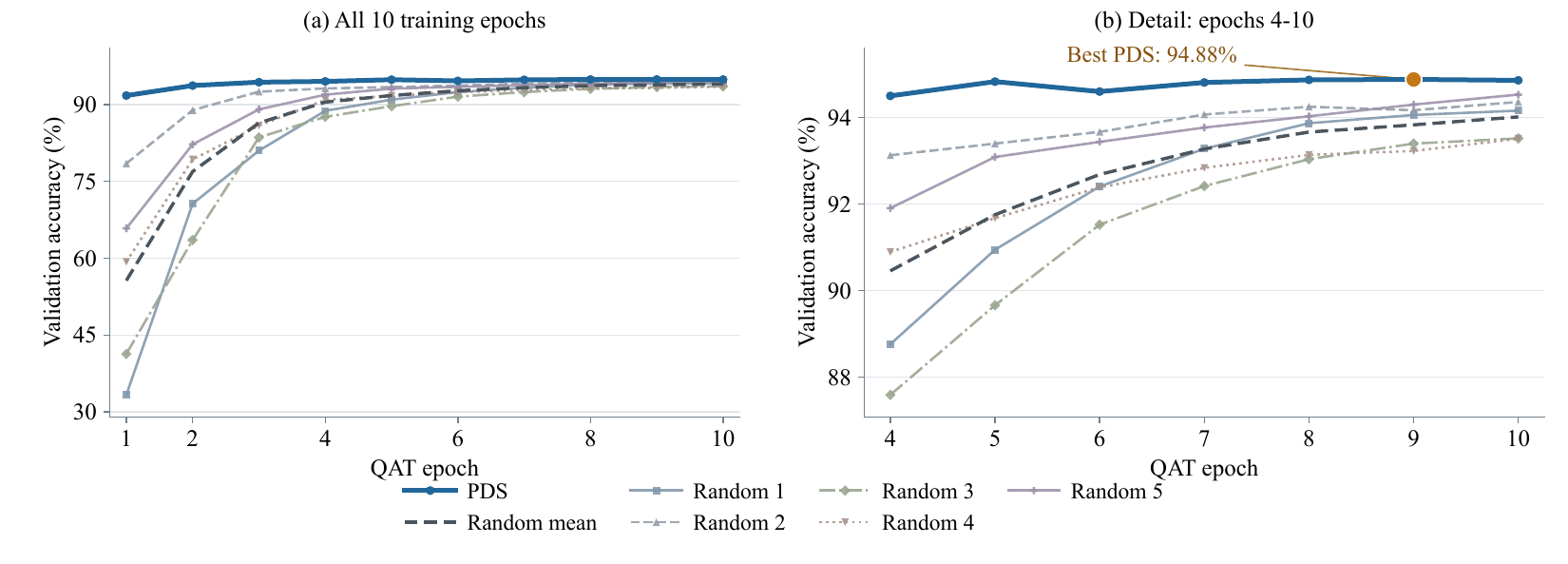
    }
    \caption{
        Equal-budget QAT comparison on GSC with Falcon-Large.
        The PDS-selected \textit{Balanced} schedule and five
        random schedules share the same delay distribution
        and modeled core latency.
        (a) Validation accuracy over all ten training epochs.
        (b) A closer view of epochs 4--10.
    }
    \label{fig:gsc_l_pds_random_qat}
\end{figure}
\clearpage
\subsection{Energy breakdown}
\label{sec: energybreak}
\paragraph{Analog-side energy.}
Using side-channel NeuroX profiling, we estimate the analog-side
inference energy of the exported Balanced models. FALCON-Medium and
Falcon-Large require an estimated $0.7436\,\mathrm{mJ}$ and
$1.0591\,\mathrm{mJ}$ per sample, respectively, including dynamic
and active-window static energy of the modeled Conv/Linear units
and their peripheral circuits
(Figure~\ref{fig:analog_energy_breakdown}). CABLC accounts for theen array-input supply branch, including
array conduction energy; ADC denotes analog-to-digital
conversion; REF supplies the ADC reference currents;
Control represents macro sequencing and control logic;
Readout (except ADC) combines current weighting (DSWCT),
input-bit accumulation (SINWP-SC), and positive/negative
current subtraction (PN-ISUB); and the tile accumulator
combines partial sums across input tiles.
\begin{figure}[H]
    \centering
    \includegraphics[width=\linewidth]{
        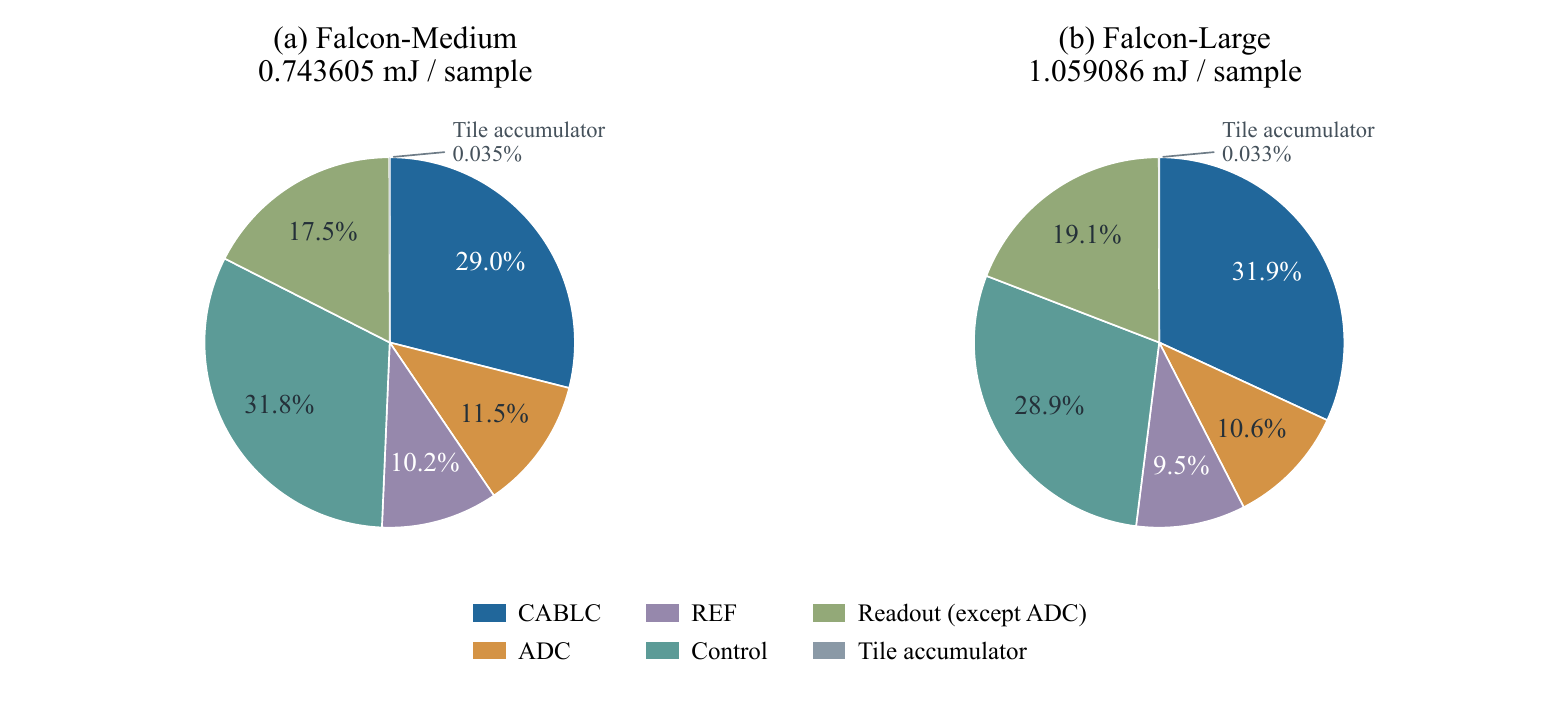
    }
    \caption{
        Analog energy breakdown.
    }
    \label{fig:analog_energy_breakdown}
\end{figure}

\paragraph{Digital attention energy estimation.}
We estimate the energy of $QK^\top$ and $AV$ using a
$32\times32$ output-stationary integer array.
The array is synthesized with Synopsys Design Compiler and
the standard-cell library at the TT corner,
$0.65\,\mathrm{V}$, $25^\circ\mathrm{C}$, and $100\,\mathrm{MHz}$.

\begin{figure}[H]
    \centering
    \includegraphics[width=\textwidth]
    {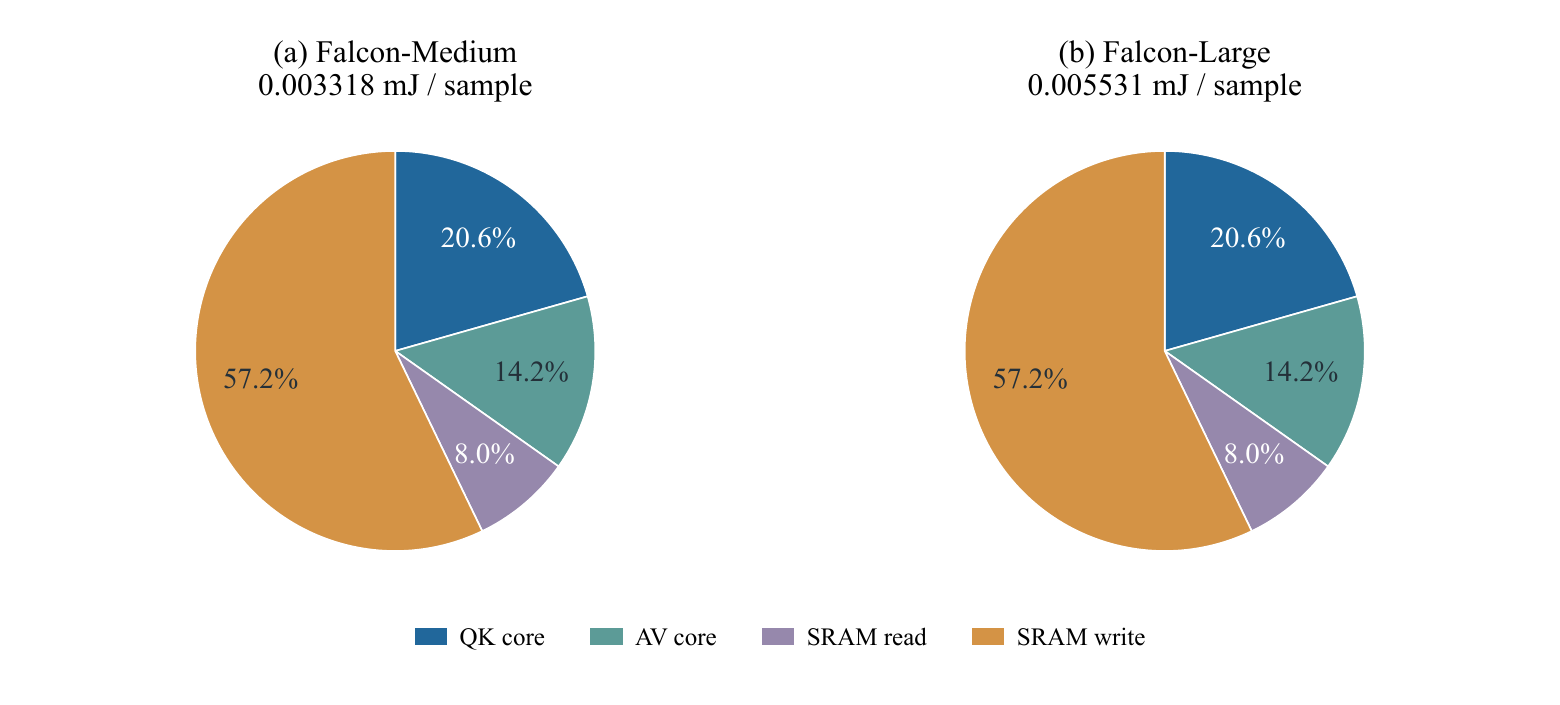}
    \caption{
      Digital energy breakdown.
        Core energy includes array arithmetic, internal registers,
        control, and leakage during the evaluation windows.
    }
    \label{fig:digital_energy_breakdown}
\end{figure}

\clearpage

\subsection{Q/K-prefix ablation}
\label{sec:results_prefix}

The main results use the full Q/K prefix, $p=7$.
We further reduce $p$ while keeping $T=7$ and the Falcon-Medium
Balanced delay schedule fixed.
The $p=7$ setting is used as a fixed reference and receives no
additional training.
For each $p<7$, we independently start from this reference
checkpoint and jointly adapt the permitted network and neuron
parameters for ten epochs.

Table~\ref{tab:qk_prefix_ablation} shows that a shorter Q/K prefix
can further reduce latency, but aggressive truncation is harder
to recover from.
Reducing $p$ from seven to six lowers core latency from
$119.64$ to $114.00\,\mu\mathrm{s}$, a $4.71\%$ reduction.
After adaptation, test accuracy reaches $96.33\%$, compared
with $96.31\%$ for the full-prefix reference.
With $p=5$, latency is reduced by $9.43\%$ while test accuracy
remains at $96.13\%$.
At $p=4$, adaptation recovers test accuracy from $84.65\%$
to $95.92\%$ at $102.72\,\mu\mathrm{s}$.
More aggressive reduction is harder to recover:
at $p=3$, adaptation raises test accuracy from $71.04\%$
to $91.15\%$, but it remains $5.16$ percentage points below
the full-prefix reference.
These results show that Q/K prefix reduction provides an
additional latency--accuracy trade-off.

\begin{table}[H]
\centering
\begin{threeparttable}
\caption{Q/K prefix ablation on GSC with Falcon-Medium.}
\label{tab:qk_prefix_ablation}
\small
\setlength{\tabcolsep}{4.5pt}
\renewcommand{\arraystretch}{1.08}

\begin{tabular}{@{}lrrrrr@{}}
\toprule
\shortstack{Q/K prefix\\$p$} &
\shortstack{Core latency\\($\mu$s)} &
\shortstack{Before val.\\(\%)} &
\shortstack{After val.\\(\%)} &
\shortstack{Before test\\(\%)} &
\shortstack{After test\\(\%)} \\
\midrule
7 (reference)
& 119.64 & 96.4332 & -- & 96.3108 & -- \\
6
& 114.00 & 95.9523 & 96.2930 & 96.1654 & 96.3289 \\
5
& 108.36 & 94.3192 & 96.1527 & 94.4480 & 96.1290 \\
4
& 102.72 & 85.0215 & 95.8120 & 84.6524 & 95.9200 \\
3
& 97.08 & 72.1872 & 91.6942 & 71.0404 & 91.1495 \\
\bottomrule
\end{tabular}

\end{threeparttable}
\end{table}

\section{Dataset and model architecture}
\subsection{Data preprocessing}
GSC contains 84,843 training, 9,981 validation, and 11,005 test
recordings. SSC contains 75,466 training, 9,981 validation, and
20,382 test recordings.
For GSC, we use Mel-spectrogram-based preprocessing,
similar to that used in Soul~\citep{du2026benchmarkingspikingneuralnetworks},
with the specific configuration described below. We crop or right-zero-pad each 16-kHz mono waveform
to one second, using random cropping during training and center cropping during evaluation when necessary. We compute a power spectrogram using a 480-point FFT, a 480-sample Hann window, a 160-sample hop, and no centered padding. We apply 64 Mel filters spanning 20~Hz to 8~kHz
and take the natural logarithm, yielding 98 acoustic frames with 64 features each.
Each log-Mel spectrogram is normalized using its own mean and standard deviation computed across all time-frequency entries.
For SSC, we use 10-ms time bins and sum every five adjacent cochlear channels, reducing 700 channels to 140, as in the channel reduction used by SpikCommander
\citep{wang2026spikcommander}. 

\subsection{Model Architecture}
\label{app:model_architecture}

Falcon-Medium and Falcon-Large share the same
architecture and Transformer dimensions, differing only in
the number of Transformer blocks: three for Medium and five
for Large. We evaluate Medium on both GSC and SSC, and Large
on GSC.

\paragraph{Input encoder and stem.}
The input encoder uses a $5\times5$ convolution with
$1\rightarrow32$ channels and stride $(2,1)$, followed by
batch normalization, rectification, and spike encoding.
The pipelined stem then contains two convolutional delay-controlled IF stages
and two fully connected (FC) delay-controlled IF stages.
The two convolutions use $3\times3$ kernels, with channel
dimensions $32\rightarrow48$ and $48\rightarrow48$, and
strides $(2,1)$ and $(1,1)$, respectively.
The strides are specified along the feature and temporal axes;
thus, the temporal token count is preserved.
After flattening the channel and feature axes at each token,
the two FC layers project the features to 160 dimensions
and then apply a $160\rightarrow160$ transformation.
The first FC layer has 768 input dimensions for GSC and
1680 for SSC, corresponding to input feature dimensions
of 64 and 140, respectively.

\paragraph{Transformer network and classifier.}

Each Transformer block uses a hidden dimension of 160 and
five attention heads, with 32 dimensions per head.
The attention module contains separate
$160\rightarrow160$ query, key, and value projections,
followed by a $160\rightarrow160$ output projection.
Attention uses ConSmax in place of softmax.
The feed-forward network consists of two FC layers,
$160\rightarrow320\rightarrow160$
The attention and feed-forward sublayers each have a residual
connection followed by batch normalization and delay-controlled IF firing.
After the final block, spike counts are decoded and batch
normalized, followed by temporal mean pooling and a
$160\rightarrow35$ classifier.
For variable-length SSC inputs, pooling excludes padded tokens.

\begin{table}[H]
\centering
\caption{Architecture summary of the two Falcon variants.
FC counts include the separate Q/K/V and output projections
in every attention module, both feed-forward layers, the
two stem FC layers, and the classifier.}
\label{tab:falcon_architecture}
\small
\begin{tabular}{lcc}
\toprule
Component & Falcon-Medium & Falcon-Large \\
\midrule
Input-encoder convolutions & 1 & 1 \\
Pipelined-stem convolutions & 2 & 2 \\
Pipelined-stem FC layers & 2 & 2 \\
Transformer blocks & 3 & 5 \\
Hidden dimension & 160 & 160 \\
Attention heads & 5 & 5 \\
FFN intermediate dimension & 320 & 320 \\
Classifier output classes & 35 & 35 \\
Total convolutional layers & 3 & 3 \\
Total FC layers & 21 & 33 \\
\bottomrule
\end{tabular}
\end{table}

\paragraph{Layer-wise firing delays.}
Table~\ref{tab:commitment_schedules} lists the firing delays for
the \textit{Fastest}, \textit{Balanced}, and \textit{Accurate}
schedules.
Each delay is assigned at an output spike boundary.
\textit{Fastest} uses $\delta=1$ at all searchable boundaries,
while \textit{Balanced} and \textit{Accurate} add waiting only
at the listed boundaries.
Conv1 and Conv2 denote the two post-encoder stem convolutions.
Within each Transformer block, AttnOut denotes the attention
output projection, and FFN1 and FFN2 denote the two feed-forward
linear layers.
Residual1 and Residual2 denote the normalized outputs of the
attention and feed-forward residual additions, respectively.
Blocks are numbered from zero.
All nine configurations use $T=7$ rate-coding slots and the
full Q/K prefix, $p=T$.
Value and Context use fixed full-lookahead delays ($\delta=T$)
and are excluded from PDS.
The input encoder and final classifier have no searchable
firing delays.

\begin{table}[H]
    \centering
\caption{
Layer-wise firing delays.
Only searchable boundaries with $\delta=2$ or $\delta=3$
are listed; all other searchable boundaries use $\delta=1$.
A dash indicates that no searchable boundary uses the
corresponding delay.
Transformer blocks are zero-indexed.
}
    \label{tab:commitment_schedules}
    \small
    \setlength{\tabcolsep}{5pt}
    \renewcommand{\arraystretch}{1.15}
    \begin{tabularx}{\textwidth}{@{}llXX@{}}
        \toprule
        Model & Schedule
        & Boundaries with $\delta=2$
        & Boundaries with $\delta=3$ \\
        \midrule

        GSC-Medium & Fastest
        & --- & --- \\

        & Balanced
        & Block 0: FFN1; Block 2: AttnOut
        & Block 2: Residual2 \\

        & Accurate
        & Stem Conv2
        & Blocks 0 and 1: FFN1;
          Block 2: AttnOut and Residual2 \\

        \addlinespace

        GSC-Large & Fastest
        & --- & --- \\

        & Balanced
        & Block 2: Residual1; Block 3: FFN1
        & Block 4: Residual2 \\

        & Accurate
        & Blocks 0, 1, 3, and 4: FFN1;
          Block 2: Residual1;
          Blocks 3 and 4: AttnOut
        & Stem Conv2;
          Block 2: FFN1;
          Block 4: Residual2 \\

        \addlinespace

        SSC-Medium & Fastest
        & --- & --- \\

        & Balanced
        & ---
        & Stem Conv1 and Conv2;
          Block 2: FFN2 \\

        & Accurate
        & Block 1: Residual1
        & Stem Conv1 and Conv2;
          Block 0: FFN1;
          Block 2: FFN2 and Residual2 \\

        \bottomrule
    \end{tabularx}
\end{table}

\clearpage
\section{Example pipeline configuration}

We illustrate a configuration selected by PDS for Falcon-Medium on GSC.
The model has four pipelined stem stages and three Transformer blocks,
with $T=p=7$. PDS searches delays $\delta \in \{1,2,3\}$ at 19 IF
stages. Q/K use the full input prefix, while Value and Context use
Full-lookahead.

The selected \emph{Balanced} schedule increases delays at only three
stages: $\delta=2$ at Block~0's FFN1 and Block~2's attention output
projection, and $\delta=3$ at Block~2's second residual.
All other searchable stages use $\delta=1$.

Figure~\ref{fig:gsc_middle_balanced_pipeline_detail} shows the overlap
across modules. Block~0 starts at $7.64\,\mu\mathrm{s}$, before the
stem produces its last output slot at $18.92\,\mu\mathrm{s}$.
Q/K/V projections run in parallel, while AV waits for both the
attention map and the complete Value count.
Under spatial mapping, the modeled core latency is
$119.64\,\mu\mathrm{s}$, compared with $130.86\,\mu\mathrm{s}$
for the matched QNN, an $8.6\%$ reduction.
\begin{figure}[H]
    \centering
    \includegraphics[width=1\linewidth]{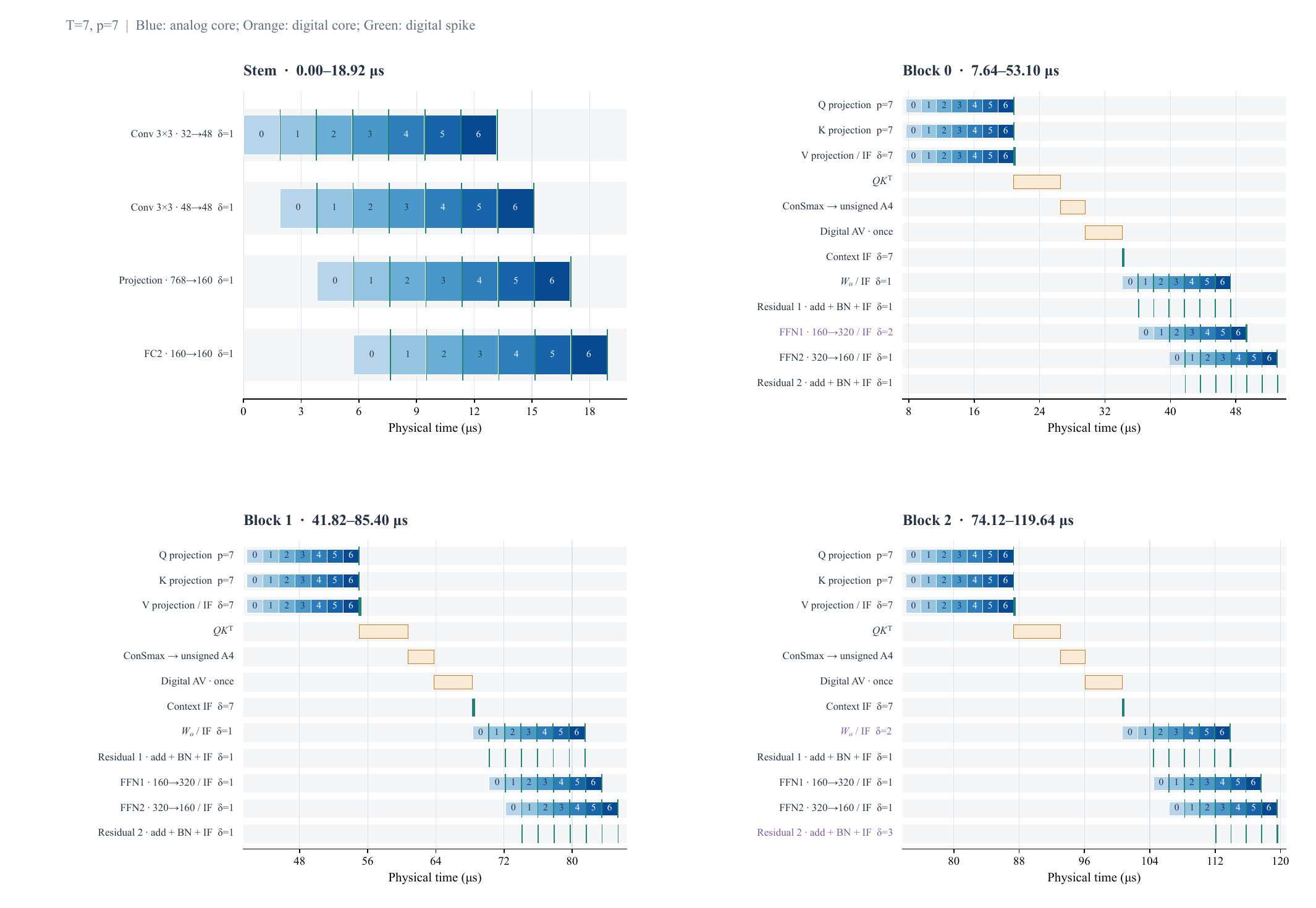}
\caption{Latency Example of Falcon-Medium with balanced schedule
}
\label{fig:gsc_middle_balanced_pipeline_detail}
\end{figure}

\clearpage
\section{Proof of Theorem~\ref{thm:local_count}}
\label{app:proof_local_count}

\begin{proof}
We omit the neuron and layer indices.
Let
\[
J(t)=\frac{I_0}{T}+I(t),
\qquad
P(r)=V(-1)+\sum_{t=0}^{r}J(t).
\]
After receiving input slots $0,\ldots,r$ and emitting $n$ spikes,
the membrane is
\[
P(r)-n\theta,
\]
where $\theta>0$.

\paragraph{Adjacent delays.}
First compare $\delta$ and $\delta+1$, with $\delta<T$.
Their decision prefixes are
\[
\begin{aligned}
\delta:\quad&
\delta-1,\delta,\ldots,T-2,
\underbrace{T-1,\ldots,T-1}_{\delta\text{ times}},
\\
\delta+1:\quad&
\delta,\delta+1,\ldots,T-2,
\underbrace{T-1,\ldots,T-1}_{\delta+1\text{ times}}.
\end{aligned}
\]
Thus, the second execution removes the first decision and adds one
full-input decision at the end.
The other $T-1$ decisions use the same prefixes in the same order.

Before they begin, the $\delta$ execution has emitted zero or one spike,
while the $\delta+1$ execution has emitted none.
If their counts are equal, their membranes and firing decisions are
equal.
If the $\delta$ execution leads by one spike, its membrane is exactly
$\theta$ lower.
It cannot fire unless the other execution also fires.
Its lead therefore stays at zero or one after every common decision.

Let
\[
n=\widehat q(\delta)
\]
and let $n'$ be the count for delay $\delta+1$ before its last decision.
Then
\[
n\in\{n',n'+1\},
\qquad
0\leq n'\leq T-1.
\]

\paragraph{The final decision.}
The full-lookahead count is
\[
Q:=\widehat q_{\mathrm{full}}
=
\min\left\{
T,\;
\max\left\{
0,\;
\left\lfloor
\frac{P(T-1)}{\theta}
\right\rfloor
\right\}
\right\}.
\]
Since $n'\leq T-1$, a final full-input decision emits a spike exactly
when $n'<Q$. Therefore,
\[
\widehat q(\delta+1)
=
n'+\mathbf 1[n'<Q].
\]

If $n'<Q$, then
\[
n
\leq
\widehat q(\delta+1)=n'+1
\leq Q.
\]
If $n'\geq Q$, then
\[
Q
\leq
\widehat q(\delta+1)=n'
\leq n.
\]
Thus, the new count lies between the old count and $Q$, and changes
by at most one.
Consequently,
\[
0\leq e(\delta)-e(\delta+1)\leq1.
\]

\paragraph{General delays and tightness.}
Summing the adjacent-delay inequality gives
\[
0
\leq
e(\delta)-e(\delta')
\leq
\delta'-\delta.
\]
Since $e(T)=0$, setting $\delta'=T$ gives
\[
e(\delta)\leq T-\delta.
\]

To show tightness, choose $V(-1)=\theta/2$ and let all input current
arrive at slot $T-1$, with
\[
T\theta<P(T-1)<(T+1)\theta.
\]
Every earlier decision emits no spike.
Delay $\delta$ leaves exactly $\delta$ decisions after the final current
arrives, and all of them emit a spike.
Hence,
\[
\widehat q(\delta)=\delta,
\qquad
Q=T,
\qquad
e(\delta)=T-\delta.
\]
Both upper bounds are attained.
\end{proof}

\section{Proof of Proposition~\ref{prop:accuracy_nonmonotone}}
\label{app:proof_accuracy_nonmonotone}

\begin{proof}
Fix any
\[
1\leq\delta_1<\delta_2<\delta_3\leq T.
\]
We construct a network with five input channels, three IF neurons
in layer $l$, and two IF neurons in layer $l+1$.
All IF neurons follow Algorithm~\ref{alg:rate_if}, with threshold
$1$ and initial membrane potential $1/2$.
All IF-layer biases and static currents are zero.
Only the delay $\delta$ of layer $l$ varies; the delay of
layer $l+1$ is fixed at $1$.

\paragraph{Fixed weights and inputs.}
We represent an input sample by an integer code vector
$\mathbf c=(c_1,\ldots,c_5)^\top$, where
$c_i\in\{0,\ldots,T\}$ specifies the number of input spikes
in channel $i$.
Each channel places these spikes in its first $c_i$ slots:
\[
s_i^{\mathrm{in}}(t;\mathbf c)
=
\mathbf 1[t<c_i],
\qquad
t=0,\ldots,T-1.
\]
Thus, $c_i$ is an integer count, whereas
$s_i^{\mathrm{in}}(t;\mathbf c)$ is a binary spike.

We use zero input baselines, unit scales, and fixed weights
\[
W^l
=
\begin{pmatrix}
1&0&0&0&0\\
0&1&-1&0&0\\
0&0&0&1&-1
\end{pmatrix},
\qquad
W^{l+1}
=
\begin{pmatrix}
1&-1&0\\
1&0&-1
\end{pmatrix},
\]
where rows index output neurons and columns index input channels.

The two input samples are
\[
\mathbf c^{(1)}
=
(1,\delta_2,\delta_2-1,\delta_3,\delta_3-1)^\top,
\qquad
\mathbf c^{(0)}
=
(1,1,0,0,0)^\top,
\]
with class labels $1$ and $0$, respectively.
The superscripts identify the class labels.
All entries lie in $\{0,\ldots,T\}$.
Once $\delta_1,\delta_2,\delta_3$ are chosen, both input samples
remain fixed as $\delta$ varies.

For an input $\mathbf c$, let
$\widehat q_j^{l+1}(\mathbf c;\delta)$ denote the final spike
count of neuron $j$ in layer $l+1$ when layer $l$ uses delay
$\delta$.
The fixed readout predicts
\[
F_\delta(\mathbf c)
=
\mathbf 1\!\left[
\widehat q_1^{l+1}(\mathbf c;\delta)
-
\widehat q_2^{l+1}(\mathbf c;\delta)
+
\frac12
>0
\right].
\]
The readout weights and bias do not change with $\delta$.
Within each input case below, we omit the input argument
from the currents, spikes, and counts.

\paragraph{The label-$1$ input.}
For $\mathbf c^{(1)}$, the currents entering layer $l$ are
\[
J_1^l(t)=\mathbf 1[t=0],
\qquad
J_2^l(t)=\mathbf 1[t=\delta_2-1],
\qquad
J_3^l(t)=\mathbf 1[t=\delta_3-1].
\]
This follows from the fixed weights and the identity
\[
\mathbf 1[t<a]-\mathbf 1[t<a-1]
=
\mathbf 1[t=a-1]
\]
for integer $t$ and $a$.

Each neuron receives one unit of current.
At the first decision that includes this input, its membrane
potential is $3/2$, so it emits one spike.
Soft reset returns the membrane potential to $1/2$.
Since no further current arrives, it emits no additional spikes.

A unit current arriving at input slot $a$ is first included
in output decision
\[
k=\max\{a-\delta+1,0\}.
\]
The output spikes are therefore
\[
\begin{aligned}
s_1^l(k;\delta)
&=
\mathbf 1[k=0],\\
s_2^l(k;\delta)
&=
\mathbf 1[k=\max\{\delta_2-\delta,0\}],\\
s_3^l(k;\delta)
&=
\mathbf 1[k=\max\{\delta_3-\delta,0\}].
\end{aligned}
\]
Here, $k$ is a logical output-slot index, not physical time.
The final count vector is $(1,1,1)$ for every
$\delta\in\{1,\ldots,T\}$, including Full-lookahead.
Thus, every neuron in layer $l$ has zero local count error.

At layer $l+1$, the two neurons receive
\[
J_1^{l+1}(k)
=
s_1^l(k;\delta)-s_2^l(k;\delta),
\qquad
J_2^{l+1}(k)
=
s_1^l(k;\delta)-s_3^l(k;\delta).
\]
This layer uses delay $1$ and sums the currents within each
input slot before making a firing decision.
If the positive input arrives before the negative input,
the membrane potential rises from $1/2$ to $3/2$ and the
neuron emits one spike.
The later negative current cannot remove that spike.
If both inputs arrive in the same slot, they cancel before
the firing decision, and the neuron emits no spike.

The final counts at layer $l+1$ are therefore
\[
\left(
\widehat q_1^{l+1},
\widehat q_2^{l+1}
\right)
=
\begin{cases}
(1,1), & \delta<\delta_2,\\
(0,1), & \delta_2\leq\delta<\delta_3,\\
(0,0), & \delta\geq\delta_3.
\end{cases}
\]
The readout scores are $1/2$, $-1/2$, and $1/2$,
respectively.
Hence,
\[
F_\delta(\mathbf c^{(1)})
=
\begin{cases}
1, & \delta<\delta_2,\\
0, & \delta_2\leq\delta<\delta_3,\\
1, & \delta\geq\delta_3.
\end{cases}
\]

\paragraph{The label-$0$ input.}
For $\mathbf c^{(0)}$, the first two neurons in layer $l$
each receive one unit of current at input slot $0$,
and the third receives none.
For every delay, their output spikes are
\[
s_1^l(k;\delta)
=
s_2^l(k;\delta)
=
\mathbf 1[k=0],
\qquad
s_3^l(k;\delta)=0.
\]
Their final counts are $(1,1,0)$, equal to the
full-lookahead counts.

At layer $l+1$, the first neuron's inputs cancel in slot $0$,
while the second neuron receives one positive unit of current.
The final count vector is therefore $(0,1)$ for every delay.
The readout score is $-1/2$, giving
\[
F_\delta(\mathbf c^{(0)})=0.
\]
This input is always classified correctly.

\paragraph{Accuracy.}
Take the fixed dataset
\[
\mathcal D
=
\left\{
(\mathbf c^{(1)},1),
(\mathbf c^{(0)},0)
\right\}.
\]
Every neuron in layer $l$ has zero local count error on both
inputs for every $\delta\in\{1,\ldots,T\}$.
However, the accuracy is
\[
\operatorname{Acc}_{\mathcal D}(\delta)
=
\begin{cases}
1,
& \delta<\delta_2\ \text{or}\ \delta\geq\delta_3,\\
\frac12,
& \delta_2\leq\delta<\delta_3.
\end{cases}
\]
Hence,
\[
\operatorname{Acc}_{\mathcal D}(\delta_1)
=
1
>
\operatorname{Acc}_{\mathcal D}(\delta_2)
=
\frac12
<
\operatorname{Acc}_{\mathcal D}(\delta_3)
=
1.
\]
Since $T\geq\delta_3$, we also have
$\operatorname{Acc}_{\mathcal D}(T)=1$.

All IF membrane potentials at firing decisions are half-integers,
since the initial potential is $1/2$ and all currents and resets
are integers.
They therefore differ from the threshold $1$ by at least $1/2$.
The result does not depend on threshold ties.
\end{proof}

\section{Latency for related works }
\label{app:latency_reconstruction}

\subsection{Digital Latency Estimation for SpikCommander}
\label{app:spikcommander_latency}

We estimate SpikCommander's modeled core latency from its
released computation graph rather than from timestep count alone.
We use the same timing boundary as FALCON.
SpikCommander uses 16 shared $32\times32$ systolic arrays,
whereas FALCON uses five arrays of the same size.
The array counts follow the respective numbers of attention
heads, giving SpikCommander more digital arrays than FALCON.

\paragraph{Hardware and cost model.}
We map SpikCommander to 16 shared $32\times32$ systolic arrays and
16 shared 32-lane vector engines running at $100\,\mathrm{MHz}$ same with our Falcon setting.
The hardware resources are fixed for all layers.
For a GEMM $(m\times k)(k\times n)$, the cycle count is
\[
C_{\mathrm{GEMM}}(m,k,n)
=
\left\lceil\frac{m}{32}\right\rceil
\left\lceil\frac{n}{32}\right\rceil k + 62 ,
\]
where the last 62 cycles account for filling and draining the systolic
array.
Element-wise operations, reductions, BN, LIF, residual additions, and
grouped convolutions are also included using the same 32-lane digital
model. The acoustic sequence is processed in groups of at most 32 frames,
matching the 32 rows of the systolic array.
Output-channel tiles are distributed across the available arrays.

SpikCommander's MSTASA computes per-frame Q/K/V features together with
local and global temporal gates.
The local gate for frame $t$ uses a radius-20 window,
\[
\mathbf g_t^{\mathrm{loc}}
=
\mathrm{LIF}\!\left(
\beta_{\mathrm{loc}}
\sum_{u\in[t-20,t+20]}
(\mathbf Q_u+\mathbf K_u)
\right),
\]
whereas the global gate uses the complete sequence,
\[
\mathbf g^{\mathrm{glob}}
=
\mathrm{LIF}\!\left(
\beta_{\mathrm{glob}}
\sum_{u=0}^{L-1}
(\mathbf Q_u+\mathbf K_u)
\right).
\]
The attention output is
\[
\mathbf A_t
=
\mathbf g_t^{\mathrm{loc}}\odot\mathbf V_t
+
\mathbf g^{\mathrm{glob}}\odot\mathbf V_t .
\]

Therefore, the local Q/K path and the convolutional V path can start
before the complete sequence is ready, while the global gate must wait
for all Q/K frames in each block. Each operation starts once its required inputs are ready and a hardware engine is available. This allows neighboring layers and blocks to overlap.

 Thus, after optimization, the estimated latencies under different timestep settings are shown below: 
For the 2L-16-256 model with $T=100$, the second block can start processing
early acoustic-frame groups before the first block has fully finished,
resulting in a latency of $940.70\,\mu\mathrm{s}$.
For the 1L-16-256 model with the same $T=100$, the latency is
$487.97\,\mu\mathrm{s}$.
We additionally evaluate the 2L-16-256 model with a shorter temporal
sequence of $T=50$, which reduces the latency to
$645.63\,\mu\mathrm{s}$.

\subsection{Digital Latency Estimation for SpikeSCR}
\label{app:spikescr_latency}

We use the same digital cost model and scheduling method as for SpikCommander: 16 shared $32\times32$ systolic arrays and
16 shared 32-lane vector engines at $100\,\mathrm{MHz}$.  
We reconstruct the 1L-16-256 model with one SGLE block,
16 attention heads, hidden dimension 256, and FFN dimension 1024.

Q, K, and V in Spikescr are produced by Linear, BN, and LIF operations. Q and K then pass through RoPE and a second LIF.
For each of the 16 heads,
\[
Q,K,V\in\mathbb{R}^{L\times16}.
\]
For a query-row group $\mathcal R$, with $|\mathcal R|\leq32$,
SpikeSCR computes
\[
S_{\mathcal R}=Q_{\mathcal R}K^\top,
\qquad
O_{\mathcal R}=S_{\mathcal R}V.
\]
A QK task waits for its query rows and the complete K sequence,
while the corresponding AV task additionally waits for the complete
V sequence. Different heads and query-row groups can overlap when
their inputs and hardware engines are ready. Using the GEMM model defined above, one 32-column QK tile costs
$C_{\mathrm{QK}}
=
C_{\mathrm{GEMM}}(|\mathcal R|,16,32)
=
16+62
=
78
$
cycles. The corresponding AV task costs
$
C_{\mathrm{AV}}
=
C_{\mathrm{GEMM}}(|\mathcal R|,L,16)
=
L+62
$
cycles.

As in SpikCommander, an operation starts once both its required inputs
and a hardware engine are available. This allows neighboring modules
and acoustic-frame groups to overlap while preserving the full-K,
full-V, convolution-window, neuron-state, and residual dependencies.

For the 1L-16-256 SpikeSCR model with $T=40$ , the resulting latency is $278.47\,\mu\mathrm{s}$.
With a longer temporal sequence of $T=100$, the
latency increases to $468.11\,\mu\mathrm{s}$.

\subsection{Digital Latency Estimation for BSO Spiking VGG-11}
\label{app:spiking_vgg_latency}

We use the same digital cost model and scheduler as for
SpikCommander: 16 shared $32\times32$ systolic arrays and
16 shared 32-lane vector engines at $100\,\mathrm{MHz}$.
We use the released \texttt{online\_spiking\_vgg11\_ws} structure
with the GSC input size of $32\times32$ and $T=4$.
The neuron settings follow the public training code:
LIF neurons with $\tau=2$, threshold 1, and soft reset.
These settings define our reconstruction.
We exclude the first convolution and its spike-generation step.
The four encoded feature maps, each of size $64\times32\times32$,
are available at time zero.
The timed network contains the remaining seven convolutions,
three average-pooling layers, LIF updates, and output scaling.
Under the same scheduling policy, a task starts when its inputs
and an engine are ready. Convolutions wait for their required
spatial neighborhoods, and each neuron update waits for its
previous-step state.
The resulting modeled core latency is $1714.88\,\mu\mathrm{s}$.


\end{document}

%% file: math_commands.tex
\usepackage{amsmath,amsfonts,bm}

\def\eqref#1{equation~\ref{#1}}

\def\1{\bm{1}}

\DeclareMathAlphabet{\mathsfit}{\encodingdefault}{\sfdefault}{m}{sl}
\SetMathAlphabet{\mathsfit}{bold}{\encodingdefault}{\sfdefault}{bx}{n}

